\documentclass{article} 
\usepackage{iclr2027_conference,times}

\usepackage{amsmath,amsfonts,bm}

\def\eqref#1{equation~\ref{#1}}

\def\1{\bm{1}}

\DeclareMathAlphabet{\mathsfit}{\encodingdefault}{\sfdefault}{m}{sl}
\SetMathAlphabet{\mathsfit}{bold}{\encodingdefault}{\sfdefault}{bx}{n}

\usepackage{amsmath,amssymb,amsthm}
\usepackage{graphicx,booktabs,array,tabularx}
\usepackage{longtable}
\usepackage{float}
\usepackage{placeins,needspace}
\usepackage{xcolor,colortbl}

\title{\raggedright SetCommit: Coordinated Token Commitments for Few-Step Diffusion Language Model Decoding}

\author{Yuannuo Feng$^{1}$ \quad Yubing Ye$^{2}$ \quad Wenshuai Yao$^{3}$ \quad Yizhe Chen$^{1}$ \\
Ngai Wong$^{2}$ \quad Wenyong Zhou$^{2}$ \quad Wang Kang$^{1}$ \\
$^{1}$Beihang University \quad $^{2}$The University of Hong Kong \quad $^{3}$Peking University \\
\texttt{\{ynfeng, yizhechen, wang.kang\}@buaa.edu.cn} \\
\texttt{\{taliayeyb, wenyongz\}@connect.hku.hk} \quad \texttt{nwong@eee.hku.hk} \\
\texttt{2200012654@stu.pku.edu.cn}
}

\iclrfinalcopy 

\usepackage[T1]{fontenc}
\usepackage{newunicodechar}
\newunicodechar{“}{\textquotedblleft}
\newunicodechar{”}{\textquotedblright}
\newcommand{\method}{SetCommit}
\definecolor{SetCommitBlue}{HTML}{2563EB}
\definecolor{SetCommitTeal}{HTML}{0F766E}
\definecolor{SetCommitTint}{HTML}{EEF4FD}
\definecolor{TableBand}{HTML}{E7ECF3}
\definecolor{TableHead}{HTML}{F3F5F8}
\newcommand{\gaincell}[1]{\textcolor{SetCommitTeal}{\textbf{#1}}}
\newcommand{\extracell}[1]{\cellcolor{SetCommitTeal!10}\gaincell{#1}}

\newtheorem{proposition}{Proposition}
\renewcommand{\eqref}[1]{\textup{(\ref{#1})}}
\usepackage{url}
\usepackage{hyperref}

\begin{document}

\maketitle
\lhead{Preprint.}

\begin{abstract}
Diffusion language models generate multiple token positions in parallel, but few-step decoding makes parallel commitment difficult: tokens that appear reliable individually may be jointly inconsistent, while premature commitments can prevent later context-dependent refinement.
In this paper, we introduce \textsc{SetCommit}, a training-free method for coordinated parallel token commitment. Instead of selecting tokens independently by confidence, \textsc{SetCommit} selects a commitment set that balances token reliability against local joint commitment risk, preserving conditioning opportunities for subsequent denoising steps. It further performs budget-neutral local exchanges, replacing vulnerable earlier commitments when newly revealed context supports more reliable alternatives. \textsc{SetCommit} requires no retraining or architectural modification and is compatible with existing diffusion decoding strategies.
Experiments with LLaDA-8B and Dream-7B on GSM8K, MATH-500, MBPP, and HumanEval show accuracy/pass@1 gains of up to 7.4 percentage points over native decoding at 32 model evaluations. At common quality targets, geometric-mean incremental speedups range from $1.44$ to $1.74\times$. 
Code is available at \url{https://github.com/Kissmetothemoon/SetCommit}.
\end{abstract}

\section{Introduction}
\label{sec:introduction}

Diffusion language models (DLMs) generate text by iteratively denoising multiple token positions in parallel, offering a promising alternative to autoregressive generation~\citep{DBLP:conf/nips/AustinJHTB21, DBLP:conf/nips/SahooASGMCRK24, DBLP:conf/nips/NieZYZOHZLWL25, DBLP:journals/corr/abs-2508-15487}. This parallelism can substantially reduce generation latency~\citep{DBLP:conf/iclr/DeschenauxG25}, but it also introduces a decoding decision absent from left-to-right generation: at each denoising step, which token predictions should be committed, and which positions should remain revisable until additional context becomes available?

This decision is particularly consequential in few-step decoding~\citep{DBLP:conf/iclr/LiuBNB25, DBLP:conf/iclr/XuGKNXLEV25}. Existing decoders commonly commit multiple positions according to independent confidence estimates or related token-level scores~\citep{DBLP:conf/nips/NieZYZOHZLWL25,wu2026fastdllm, DBLP:conf/nips/Ben-HamuGSNK25}. However, individually reliable predictions need not be reliable when committed together: simultaneous commitments may be mutually inconsistent, and prematurely fixing nearby uncertain positions can prevent them from benefiting from context revealed in later denoising steps~\citep{DBLP:conf/iclr/LiuBNB25, DBLP:conf/nips/IsraelBG25}. With only a limited number of updates, such unfavorable commitments leave few opportunities for recovery. Effective few-step decoding therefore requires coordinating \emph{sets} of parallel commitments, rather than merely selecting the most confident tokens independently.

Existing acceleration methods address complementary bottlenecks.
System-side methods reduce forward-pass cost through sparse or cached
computation~\citep{DBLP:conf/nips/MaYFW25,liu2026dllmcache,DBLP:conf/aaai/SongLLLHGHQ26,wu2026fastdllm},
while schedule-side methods reduce or adaptively allocate denoising
updates~\citep{DBLP:conf/iclr/DeschenauxG25,DBLP:conf/nips/Ben-HamuGSNK25,DBLP:conf/iclr/WeiZLZWQLZ26,wu2026fastdllm}.
Figure~\ref{fig:overview} locates commitment coordination at their shared
interface: model predictions must be converted into committed positions
before the decoder state is updated. Improving \emph{which} positions are
committed can therefore complement these accelerators while retaining
the host's computation mode and scheduling rule.

\begin{figure}[!t]
\centering
\includegraphics[width=\linewidth]{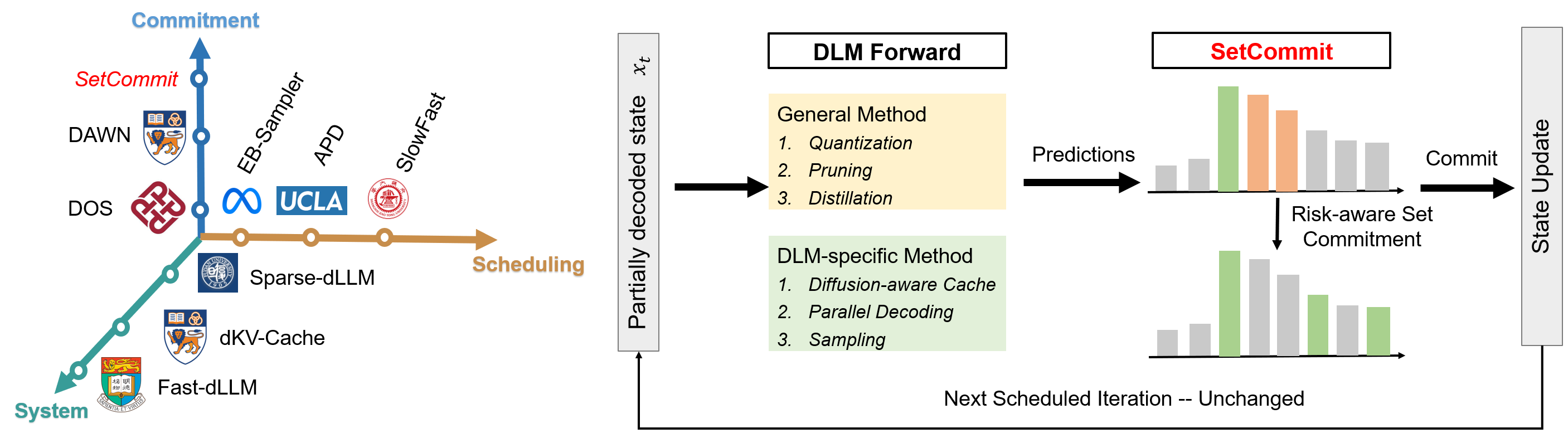}
\caption{\textbf{\textsc{SetCommit} complements system- and schedule-side diffusion decoding accelerators.}
Left: it coordinates parallel token commitments, orthogonally to reducing forward-pass computation or denoising updates.
Right: predictions enter \textsc{SetCommit}; the lower row's selected positions (green) feed State Update under the host's net quota. The return arrow follows the host scheduling rule.}
\label{fig:overview}
\end{figure}

We introduce \textsc{SetCommit}, a training-free commitment coordinator
inserted between the DLM forward pass and state update. It balances
token reliability with local joint risk to preserve future conditioning
opportunities, then uses budget-neutral exchanges to reopen vulnerable
earlier commitments. The module preserves the host's computation mode
and net commitment quota and composes with
native~\citep{DBLP:conf/nips/NieZYZOHZLWL25,DBLP:journals/corr/abs-2508-15487},
cached, scheduled~\citep{DBLP:conf/nips/Ben-HamuGSNK25}, and combined
decoders~\citep{wu2026fastdllm}.
Our contributions are threefold:
\begin{itemize}
\item We identify \emph{commitment coordination} as an orthogonal
acceleration axis for few-step DLM decoding, improving parallel
commitment decisions at the interface between model prediction and
decoder state update.

\item We propose \textsc{SetCommit}, a training-free and plug-and-play commitment coordinator. \textsc{SetCommit} selects commitment sets by jointly considering token-level reliability and local joint commitment risk, and introduces budget-neutral local exchange to revise vulnerable earlier commitments as generation context evolves.

\item Experiments with LLaDA-8B and Dream-7B on GSM8K, MATH-500, MBPP, and HumanEval show $4.9$--$7.4$ percentage-point quality gains over native decoding at 32 evaluations. Across native, fixed/adaptive Fast-dLLM, and EB-Sampler hosts, \textsc{SetCommit} achieves $1.44$--$1.74\times$ geometric-mean incremental speedups at common quality targets. 
\end{itemize}

\section{Related Work}
\label{sec:related_work}


\noindent\textbf{Efficient DLM Decoding.}
Existing work improves DLM inference efficiency along two largely complementary directions. Compute-side methods reduce the cost of repeated bidirectional forward passes through delayed or approximate KV caching, cache reuse, dynamic eviction, or sparse attention~\citep{DBLP:conf/nips/MaYFW25,wu2026fastdllm,liu2026dllmcache, DBLP:conf/aaai/SongLLLHGHQ26}. Schedule-side methods instead reduce the number of denoising updates, or adaptively allocate update budgets according to confidence, entropy, or other uncertainty signals~\citep{DBLP:conf/iclr/DeschenauxG25, DBLP:conf/iclr/WeiZLZWQLZ26,DBLP:conf/nips/Ben-HamuGSNK25, li2026diffusion}. Confidence-aware parallel decoding further uses token-level scores to determine how many or which predictions to reveal at an update~\citep{DBLP:conf/nips/NieZYZOHZLWL25, wu2026fastdllm}. These methods primarily optimize the cost of each forward pass or the allocation of denoising steps; when selecting multiple tokens, they typically rely on token-wise criteria. In contrast, \textsc{SetCommit} addresses the orthogonal question of which positions should be committed \emph{together}. Operating between the DLM forward pass and decoder state update, it can therefore be combined with native, cached, scheduled, and jointly accelerated host decoders.

\noindent\textbf{Commitment Revision and Self-correcting Decoding.}
Another related line of work mitigates errors caused by premature predictions in masked diffusion decoding. Such approaches may remask previously revealed tokens at inference time, train models to recover from their own errors, or modify transition and sampling procedures to permit revision of visible tokens~\citep{DBLP:conf/nips/WangSSK25, DBLP:conf/icml/RutteFDOS025, DBLP:conf/nips/ZhaoSCDML25}. These methods focus on recovering after an unfavorable token-level commitment has been made, often by explicitly revisiting token values or altering the decoding process. \textsc{SetCommit} couples preventive set selection with explicit remasking of vulnerable earlier commitments. It balances reliability and local co-commitment risk, then offsets each accepted withdrawal with an additional fresh commitment to preserve the host-selected net quota. This joint selection-and-revision interface composes with host computation and scheduling mechanisms; it does not require model retraining.

\section{SetCommit: Coordinated Commitment for Parallel Decoding}
\label{sec:method}

\subsection{Problem Setup and Coordinated Commitment}
\label{sec:setup}

Let $x^{(t)}$ denote the partially decoded sequence at iteration $t$.
Within the host-editable block of length $L$, let $M_t$ and $C_t$ denote
masked and visible positions, respectively; positions retain their
sequence indices. The host supplies current distributions $p_i^{(t)}$
and a net-commitment quota $q_t$. The module uses top-1 proposals
$\hat y_i^{(t)}=\arg\max_v p_i^{(t)}(v)$ and posteriors
$c_i^{(t)}=p_i^{(t)}(\hat y_i^{(t)})$. The quota specifies the net
decrease in masked positions, allowing some earlier commitments to be
remasked while additional positions are revealed. Below we describe the
Native/Fast policy; the separate EB configuration is specified in
Appendix~\ref{app:eb}.

Standard parallel decoding ranks positions independently by their
confidence $c_i^{(t)}$ and commits the top $q_t$. Its baseline set is

\begin{equation}
B_t = \operatorname{TopK}_{i \in M_t}\left(c_i^{(t)}, q_t\right).
\label{eq:baseline-selection}
\end{equation}

This rule is locally myopic: it evaluates each position in isolation and does not consider whether several predictions should be committed simultaneously. Consequently, independently selected positions may concentrate in a short span even when their predictions remain strongly dependent on nearby unresolved context.

\begin{figure}[t]
    \centering
    \includegraphics[width=\linewidth]{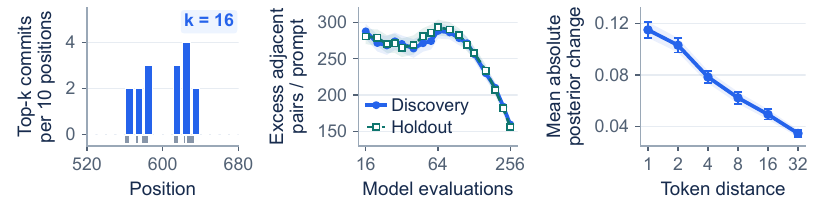}
    \caption{\textbf{Local clustering and context sensitivity on GSM8K.}
Left: counts of 16 top-$k$ commitments per 10-position bin in one state.
Center: cumulative excess adjacent pairs per prompt in the Discovery/Holdout
subsets (240 prompts each). Right: mean absolute predicted-token posterior
change after revealing a nearby prediction. Bands/bars are mean $\pm1.96$ SE;
aggregation units and metrics are detailed in Appendix~\ref{app:preliminary}.}
    \label{fig:low_step}
\end{figure}

Figure~\ref{fig:low_step} shows clustered top-$k$ positions and similar
excess-adjacency profiles in both subsets. Cumulative excess adjacency
varies nonmonotonically over 16--64 evaluations and declines toward 256.
The posterior response falls from 0.115 at distance 1 to 0.035 at
distance 32, indicating stronger sensitivity to nearby revealed context.
Together, these observations motivate coordinating simultaneous commitments
so that context-sensitive positions can use information revealed later.

\method{} balances token reliability with the local structure of the
commitment set. Its base selection preserves the host-selected quota;
a separate remask exchange reopens vulnerable earlier commitments while
maintaining the same net progress.

\subsection{Reliability-Aware Set Selection}
\label{sec:selection}

To test the effect of commitment structure on generation quality, we compare alternative commitment sets constructed from the same decoding state under the same commit count. Starting from an identical partially decoded sequence, each branch reveals a different position set chosen by one selection rule, and decoding then proceeds to completion under a shared continuation policy. All branches share the source state's token proposals, the continuation policy, and the remaining evaluation budget, isolating the choice of positions revealed at the intervention as the only degree of freedom. Because proposals, visible context, and remaining compute are held fixed across branches, systematic differences in final task accuracy can be attributed to the composition of the commitment set rather than to sampling stochasticity, budget allocation, or continuation randomness.

\begin{figure}[t]
    \centering
    \includegraphics[width=\linewidth]{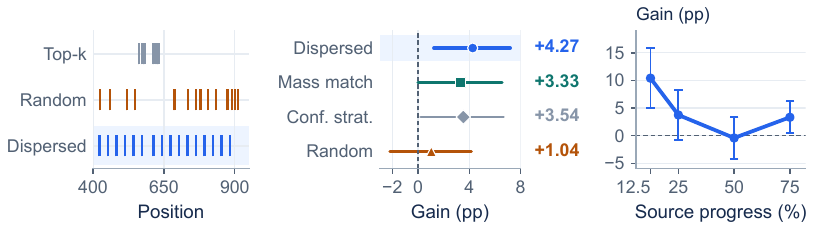}
    \caption{\textbf{Commitment structure and continuation accuracy.}
Left: 16 selected positions per rule in a shared state.
Center: paired GSM8K accuracy differences from clustered top-$k$ (pp),
averaged over four source-progress anchors at 32 evaluations.
Right: dispersed-minus-clustered gain versus source progress.
Intervals are 95\% paired bootstrap CIs over 240 prompts
(Appendix~\ref{app:preliminary}).}
    \label{fig:commitment_evidence}
\end{figure}

Figure~\ref{fig:commitment_evidence} holds the source state and commit count
fixed. Dispersed, posterior-mass-matched, and confidence-stratified rules gain
4.27, 3.33, and 3.54 pp over clustered top-$k$; random selection gains 1.04 pp
with a CI spanning zero. The dispersed effect is largest at 12.5\% progress
(10.42 pp), falls to $-0.42$ pp at 50\%, and recovers to 3.33 pp at 75\%.
The progress-dependent effects motivate balancing token reliability
with local joint risk; the figure reports the uncertainty at each anchor.

For masked position $i$, the selector measures reliability by the logit margin
\begin{equation}
m_i^{(t)}=z_{i,(1)}^{(t)}-z_{i,(2)}^{(t)},
\label{eq:margin}
\end{equation}
where $(1),(2)$ denote the two largest vocabulary logits. Min--max
normalization over $M_t$ gives reliability $\bar m_i^{(t)}\in[0,1]$;
posterior $c_i^{(t)}$ defines the baseline $B_t$ in
Equation~\ref{eq:baseline-selection}. Vocabulary-wide uncertainty is
measured by normalized entropy:
\begin{equation}
u_i^{(t)}=-\frac{1}{\log|\mathcal V|}
\sum_{v\in\mathcal V}p_i^{(t)}(v)\log p_i^{(t)}(v).
\label{eq:entropy}
\end{equation}
With mask density $d_t=|M_t|/L$, the local co-commitment risk is
\begin{equation}
a_{ij}^{(t)}=d_t\exp(-|i-j|/\tau_t)u_i^{(t)}u_j^{(t)}
\quad(i\ne j),\qquad a_{ii}^{(t)}=0.
\label{eq:pairwise-risk}
\end{equation}
This entropy-weighted locality proxy assigns larger risk to jointly
uncertain nearby predictions. The distance scale $\tau_t$ determines
how far the coordination penalty extends within the editable block.

The candidate pool contains up to $\lceil\kappa q_t\rceil$ margin-ranked
masked positions, augmented with $B_t$, where $\kappa\ge1$ controls
shortlist size. Starting from an empty set, the selector greedily adds
candidates with score
\begin{equation}
g_t(i\mid S)=\bar m_i^{(t)}-
\min\!\left\{\lambda_t\sum_{j\in S}a_{ij}^{(t)},\,\beta\right\}.
\label{eq:greedy-selection}
\end{equation}
The penalty cap $\beta\ge0$ limits how strongly risk can outweigh reliability.
For a retention ratio $\rho\in[0,1]$, the selector reserves enough remaining
slots for baseline positions to satisfy
\begin{equation}
|S_t|=q_t,\qquad |S_t\cap B_t|\ge\lceil\rho q_t\rceil.
\label{eq:baseline-retention}
\end{equation}
Baseline retention and the capped risk penalty limit departures from
reliable host decisions. Risk strength and bandwidth decrease with
progress; normalization, schedules, and default settings appear in
Appendices~\ref{app:implementation} and~\ref{app:configuration}.
\begin{figure}[t]
\centering
\includegraphics[width=\linewidth]{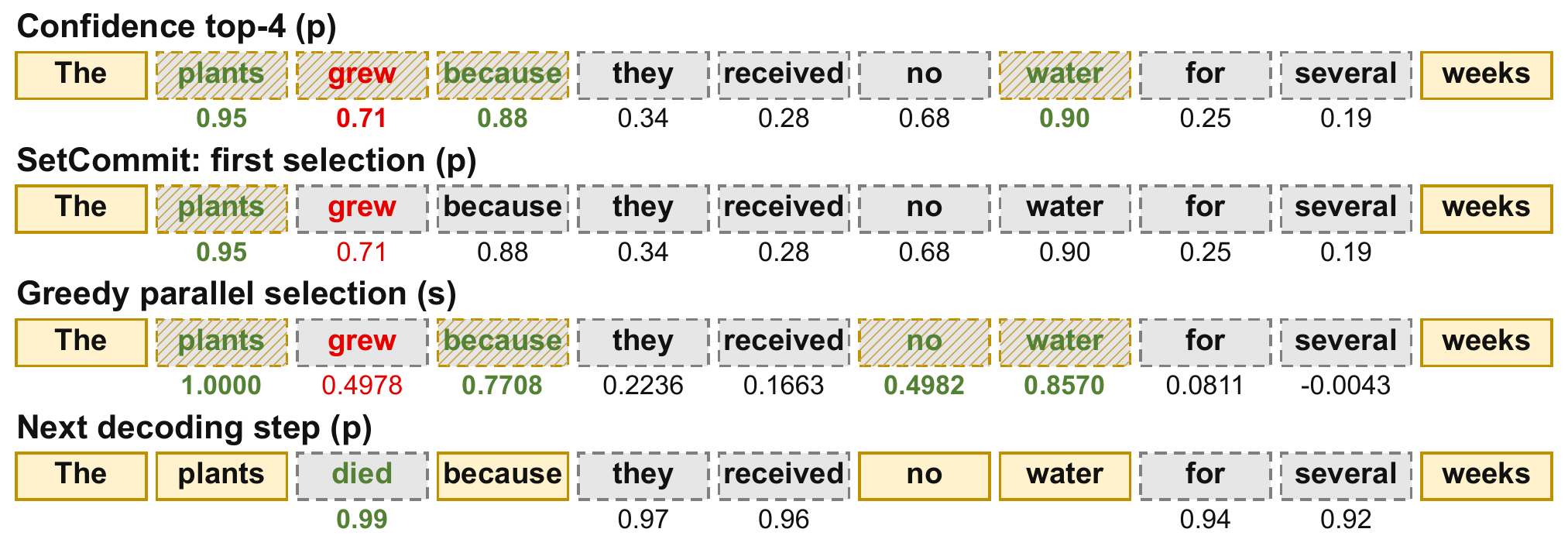}
\caption{\textbf{Coordinated selection postpones a context-sensitive commitment.}
Conceptual word-level example with $q_t=4$ and three baseline positions
retained. $p$ denotes posterior confidence and $s$ the greedy score $g_t$.
The first three rows share one forward pass; the last shows the next pass.}
\label{fig:method}
\end{figure}
In Figure~\ref{fig:method}, both sets select four positions and share
\emph{plants}, \emph{because}, and \emph{water}, satisfying the retention
constraint. Selecting \emph{no} leaves \emph{grew} deferred; the next
forward then predicts \emph{died} at that still-masked position.
The exchange operation below extends revision to previously visible tokens.

\subsection{Bounded Remasking and Host Integration}
\label{sec:integration}

New context can change the model's support for an earlier token.
Exchange withdraws vulnerable visible commitments and adds equally many
fresh commitments. For visible position $i\in C_t$, let
$o_i^{(t)}=p_i^{(t)}(x_i^{(t)})$ be the current-token posterior.
An old position is eligible only if its token is outside the current
top-$K$ predictions and $c_i^{(t)}>o_i^{(t)}+\epsilon_1$, where $K$ is
the protection rank and $\epsilon_1>0$ is a posterior-gain margin.
Exchange is active only when normalized host progress $s_t$ lies in
$[s_{\min},s_{\max}]\subseteq[0,1]$, with count capped by
\begin{equation}
k_t=\min\{|M_t|-q_t,\lfloor \alpha q_t\rfloor\}.
\label{eq:remask-cap}
\end{equation}
Here $\alpha\in[0,1]$ limits withdrawals relative to the net quota.
Eligible old positions with low current-token posteriors form $R_t$;
extra commitments $E_t\subseteq M_t\setminus S_t$ use the same
selection rule, with $|E_t|=|R_t|\le k_t$.
Each extra-token posterior must exceed its paired withdrawn-token
posterior by more than $\epsilon_1$.
With set risk $H_t(U)=\sum_{i<j,\,i,j\in U}a_{ij}^{(t)}$, acceptance
also requires
\begin{equation}
H_t\bigl((C_t\setminus R_t)\cup S_t\cup E_t\bigr)
\le H_t(C_t\cup S_t)+\epsilon_2.
\label{eq:remask-risk}
\end{equation}
Here $\epsilon_2\ge0$ is a numerical tolerance; defaults are specified
in Appendix~\ref{app:implementation}.
The implementation tries feasible exchange sizes from the cap down to
one and accepts the first passing size. Outside the progress window,
or if no candidate passes, $R_t=E_t=\varnothing$.

The host remasks $R_t$ and writes current top-1 proposals at
$A_t=S_t\cup E_t$. Withdrawn positions remain masked until a subsequent
evaluation. Thus,
\begin{equation}
\begin{aligned}
M_{t+1}&=(M_t\setminus A_t)\cup R_t,\\
|M_t|-|M_{t+1}|&=|S_t|+|E_t|-|R_t|=q_t.
\end{aligned}
\label{eq:remask-quota}
\end{equation}
This reopens earlier decisions while preserving the host-selected net
quota exactly; Appendix~\ref{app:proof} gives the accounting argument.

Selection and exchange reuse current forward-pass signals. Fixed-budget
hosts retain their quota schedule; adaptive hosts apply their acceptance
rule to the updated state, allowing subsequent quotas and realized NFE to
change. Cached hosts refresh affected positions through their backend,
with this work included in latency. The module leaves model weights and
the host stopping rule unchanged. Model-specific proposal and evaluation
protocols are specified in Appendix~\ref{app:protocol}.

\begin{table}[!t]
\centering
\caption{ \textbf{Multi-benchmark quality--cost comparison.}
$Q$: accuracy/pass@1 (\%); NFE: mean evaluations. Total uses the matched-quality
native host; mean extra averages paired-host speedups geometrically across
tasks (3-point margin). Shading marks \method{}; blue NFE marks paired
reductions. Bold Total values are column-best within each model.}
\label{tab:main}
\small
\setlength{\tabcolsep}{3.2pt}
\begin{tabular}{l*{4}{rrr}r}
\toprule
\rowcolor{TableHead}
Method & \multicolumn{3}{c}{GSM8K} & \multicolumn{3}{c}{MATH-500} &
\multicolumn{3}{c}{MBPP} & \multicolumn{3}{c}{HumanEval} & \shortstack{Mean\\extra}\\
\cmidrule(lr){2-4}\cmidrule(lr){5-7}\cmidrule(lr){8-10}\cmidrule(lr){11-13}
& $Q\uparrow$ & NFE$\downarrow$ & Total & $Q\uparrow$ & NFE$\downarrow$ & Total &
$Q\uparrow$ & NFE$\downarrow$ & Total & $Q\uparrow$ & NFE$\downarrow$ & Total &\\
\midrule
\rowcolor{TableBand}
\multicolumn{14}{l}{\textbf{LLaDA-8B}\quad\textit{Quality targets: 73.0/32.4/44.5/37.2\%}}\\
Native & 73.1 & 64 & 1.00 & 34.2 & 96 & 1.00 & 46.6 & 96 & 1.00 & 39.0 & 96 & 1.00 & --\\
\rowcolor{SetCommitTint}\hspace{0.5em}\textcolor{SetCommitBlue}{\textbf{+ \method{}}} & 74.2 & \textcolor{SetCommitBlue}{48} & 1.27 & 33.8 & \textcolor{SetCommitBlue}{48} & 1.90 & 46.1 & \textcolor{SetCommitBlue}{48} & 1.94 & 39.0 & \textcolor{SetCommitBlue}{48} & 1.89 & \extracell{1.72}\\
Fast fixed & 73.1 & 64 & 2.01 & 34.0 & 96 & 1.96 & 44.7 & 64 & 2.93 & 39.6 & 96 & 1.91 & --\\
\rowcolor{SetCommitTint}\hspace{0.5em}\textcolor{SetCommitBlue}{\textbf{+ \method{}}} & 74.6 & \textcolor{SetCommitBlue}{48} & 2.38 & 34.6 & \textcolor{SetCommitBlue}{48} & 3.72 & 46.8 & \textcolor{SetCommitBlue}{48} & 3.47 & 38.4 & \textcolor{SetCommitBlue}{48} & 3.36 & \extracell{1.47}\\
Fast adaptive & 73.6 & 62.2 & 2.05 & 32.4 & 64.8 & 2.99 & 44.7 & 58.3 & 3.22 & 39.0 & 93.4 & 1.99 & --\\
\rowcolor{SetCommitTint}\hspace{0.5em}\textcolor{SetCommitBlue}{\textbf{+ \method{}}} & 73.1 & \textcolor{SetCommitBlue}{32.3} & \textbf{3.49} & 33.4 & \textcolor{SetCommitBlue}{48} & \textbf{3.87} & 45.4 & \textcolor{SetCommitBlue}{44.9} & \textbf{3.76} & 38.4 & \textcolor{SetCommitBlue}{44.8} & \textbf{3.64} & \extracell{1.47}\\
EB-Sampler & 73.6 & 58.5 & 1.09 & 32.4 & 62.9 & 1.53 & 45.2 & 63.5 & 1.50 & 39.6 & 92.6 & 1.01 & --\\
\rowcolor{SetCommitTint}\hspace{0.5em}\textcolor{SetCommitBlue}{\textbf{+ \method{}}} & 75.7 & \textcolor{SetCommitBlue}{47.7} & 1.33 & 33.6 & \textcolor{SetCommitBlue}{44.6} & 2.12 & 45.9 & \textcolor{SetCommitBlue}{45.8} & 1.97 & 39.0 & \textcolor{SetCommitBlue}{46.4} & 1.95 & \extracell{1.44}\\
DAWN & 74.0 & 60.1 & 1.04 & 33.2 & 64.7 & 1.39 & 45.2 & 60.5 & 1.53 & 38.4 & 60.9 & 1.44 & --\\
DOS & 73.5 & 59.9 & 1.05 & 32.4 & 63 & 1.45 & 45.4 & 62.7 & 1.50 & 39.6 & 97.4 & 0.96 & --\\
\midrule
\rowcolor{TableBand}
\multicolumn{14}{l}{\textbf{Dream-7B}\quad\textit{Quality targets: 77.3/40.6/56.7/51.9\%}}\\
Native & 78.9 & 96 & 1.00 & 42.6 & 96 & 1.00 & 57.4 & 64 & 1.00 & 54.3 & 96 & 1.00 & --\\
\rowcolor{SetCommitTint}\hspace{0.5em}\textcolor{SetCommitBlue}{\textbf{+ \method{}}} & 78.6 & \textcolor{SetCommitBlue}{48} & 1.96 & 41.2 & \textcolor{SetCommitBlue}{48} & 1.92 & 58.3 & \textcolor{SetCommitBlue}{48} & 1.25 & 53.7 & \textcolor{SetCommitBlue}{48} & 1.93 & \extracell{1.74}\\
Fast fixed & 77.4 & 64 & 2.89 & 42.8 & 96 & 2.01 & 57.6 & 64 & 1.92 & 52.4 & 64 & 2.83 & --\\
\rowcolor{SetCommitTint}\hspace{0.5em}\textcolor{SetCommitBlue}{\textbf{+ \method{}}} & 78.0 & \textcolor{SetCommitBlue}{48} & 3.56 & 42.8 & \textcolor{SetCommitBlue}{48} & 3.77 & 57.6 & \textcolor{SetCommitBlue}{32} & \textbf{3.39} & 53.7 & \textcolor{SetCommitBlue}{48} & 3.35 & \extracell{1.48}\\
Fast adaptive & 77.9 & 65.3 & 2.97 & 42.4 & 92.2 & 2.09 & 57.6 & 59.3 & 2.15 & 54.3 & 93.5 & 2.11 & --\\
\rowcolor{SetCommitTint}\hspace{0.5em}\textcolor{SetCommitBlue}{\textbf{+ \method{}}} & 77.6 & \textcolor{SetCommitBlue}{31.2} & \textbf{5.47} & 41.8 & \textcolor{SetCommitBlue}{46.7} & \textbf{3.96} & 59.0 & \textcolor{SetCommitBlue}{43.9} & 2.59 & 53.7 & \textcolor{SetCommitBlue}{45.6} & \textbf{3.58} & \extracell{1.63}\\
EB-Sampler & 78.2 & 61.9 & 1.52 & 41.0 & 64.1 & 1.52 & 58.1 & 60.5 & 1.07 & 53.0 & 63.2 & 1.49 & --\\
\rowcolor{SetCommitTint}\hspace{0.5em}\textcolor{SetCommitBlue}{\textbf{+ \method{}}} & 77.7 & \textcolor{SetCommitBlue}{31.6} & 2.89 & 42.0 & \textcolor{SetCommitBlue}{44.8} & 2.10 & 57.4 & \textcolor{SetCommitBlue}{32.4} & 1.83 & 52.4 & \textcolor{SetCommitBlue}{58.6} & 1.56 & \extracell{1.47}\\
DAWN & 78.2 & 60.7 & 1.52 & 41.0 & 64.8 & 1.44 & 58.3 & 62.7 & 0.99 & 52.4 & 61.1 & 1.48 & --\\
DOS & 78.2 & 61.8 & 1.48 & 41.0 & 62.7 & 1.50 & 57.8 & 60.7 & 1.03 & 52.4 & 63.2 & 1.44 & --\\
\bottomrule
\end{tabular}
\end{table}
\section{Experiments}
\label{sec:experiments}

\subsection{Experimental Setup}
\label{sec:experimental_setup}

\paragraph{Models and benchmarks.}
We evaluate LLaDA-8B and Dream-7B
\citep{DBLP:conf/nips/NieZYZOHZLWL25,DBLP:journals/corr/abs-2508-15487} on GSM8K \citep{DBLP:journals/corr/abs-2110-14168},
MATH-500 \citep{DBLP:conf/nips/HendrycksBKABTS21, DBLP:conf/iclr/LightmanKBEBLLS24}, MBPP \citep{DBLP:journals/corr/abs-2108-07732}, and
HumanEval \citep{DBLP:journals/corr/abs-2107-03374}.
We report accuracy for mathematics and functional pass@1 for code;
LLaDA2.1/2.2 provide transfer targets.

\paragraph{Baselines and paired hosts.}
Baselines are native sampling, Fast-dLLM fixed/adaptive
\citep{wu2026fastdllm}, EB-Sampler
\citep{DBLP:conf/nips/Ben-HamuGSNK25}, DAWN \citep{DBLP:journals/corr/abs-2602-06953}, and DOS
\citep{DBLP:conf/acl/ZhouHH26}.
The first four hosts are also evaluated with \method{}, sharing models,
inputs, and scorers. For Dream, module proposals change native sampling
to deterministic top-1 predictions, so its comparisons evaluate complete
decoding pipelines (Appendix~\ref{app:protocol}). EB uses our
reimplementation and a separate module policy (Appendix~\ref{app:eb}).

\paragraph{Quality and cost.}
Each model--task pair uses a common target
$q_\star=Q_{\mathrm{ref}}-\delta$, where $Q_{\mathrm{ref}}$ is native-128
quality and $\delta=3$ percentage points. We select the lowest-latency
budget/threshold setting meeting the unrounded target; development and
confirmation protocols appear in Appendix~\ref{app:protocol}.
NFE counts actual forwards. At this target, incremental and total
speedups are $S_{\mathrm{inc}}=L_B/L_{B+\method{}}$ and
$S_{\mathrm{total}}=L_{\mathrm{native}}/L_{B+\method{}}$ for host $B$.
All metrics use the same configuration. Paired timing on matched L20
hardware includes selection, remasking, and cache refresh
(Appendix~\ref{app:analysis}).
Matched-quality comparisons measure the cost of reaching a shared
quality requirement. Fixed-NFE ablations complement this view by
measuring the quality delivered by each decoding pipeline at the same
number of model evaluations.

\subsection{Main Results}
\label{sec:native_results}
\label{sec:composed_results}

\paragraph{Fewer evaluations at a common quality target.}
In Table~\ref{tab:main}, Native+\method{} reaches all eight model--task
targets at 48 evaluations, compared with 64--96 for native decoding.
The resulting geometric-mean incremental speedups are $1.72\times$ for
LLaDA and $1.74\times$ for Dream. Thus the quality improvements translate
into a smaller evaluation budget across both mathematical reasoning and
code generation, even after charging the added coordination cost.
For LLaDA/HumanEval, both variants reach 39.0\% pass@1, while
\method{} reduces NFE from 96 to 48 and achieves $1.89\times$
speedup. This illustrates the benefit at unchanged measured quality.

\paragraph{Complementarity with accelerated hosts.}
Every paired entry in Table~\ref{tab:main} reduces NFE, including hosts
that already use caching or adaptive scheduling. Their geometric-mean
incremental speedups range from $1.44$ to $1.63\times$, supporting
commitment coordination as a useful addition to these optimizations.
The best composition remains task-dependent: on Dream/GSM8K, Fast adaptive
improves from $2.97$ to $5.47\times$ total speedup, while on Dream/MBPP,
Fast fixed+\method{} is faster than Fast adaptive+\method{}
($3.39$ versus $2.59\times$ total).

\paragraph{Where the latency savings arise.}
Figure~\ref{fig:frontiers} shows the largest separation in the low-latency
region, with paired curves approaching similar quality at larger budgets.
This pattern is consistent with coordination recovering quality when few
updates are available. At the common LLaDA/GSM8K target, the adaptive
Fast host reduces NFE from 62.2 to 32.3 and gains $1.70\times$ in latency;
the fixed Fast host reduces NFE from 64 to 48 and gains $1.19\times$.
Adaptive scheduling can turn improved intermediate predictions into
different later quotas, whereas a fixed host must select another point
on its budget grid.

\begin{figure}[!t]
\centering
\includegraphics[width=\linewidth]{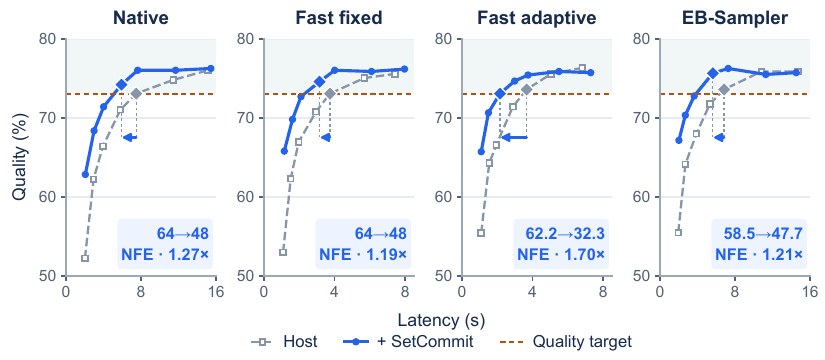}
\caption{ \textbf{Paired quality--latency curves.}
LLaDA-8B/GSM8K on four hosts. Diamonds mark selected points; dashed lines
mark the quality target. Arrows show latency savings; badges give paired
NFE and incremental speedup, using the same records as Table~\ref{tab:main}.}
\label{fig:frontiers}
\end{figure}

\subsection{Ablation Studies}
\label{sec:mechanism_results}

\paragraph{Selection, exchange, and controls.}
\label{sec:ablation_results}
Table~\ref{tab:fixed_budget} crosses selection $s$ and remasking $r$ at
32 NFE, with quality denoted by $Q_{sr}$. Selection alone outperforms
confidence-plus-spacing, stratified random, and shuffled-risk controls
in all four settings. These controls support combining reliability with
position-specific risk beyond geometric spacing or random perturbation.
The shuffled-risk control specifically tests the assignment of risk to
positions. Holding NFE fixed measures how effectively each variant uses
the same number of model evaluations.
Full \method{} gives the highest
quality in each setting, adding 4.9--6.7 pp over the host at the same NFE
for 2.4\% mean latency overhead.

\begin{table}[t]
\centering
\caption{ \textbf{Component and control ablations.}
Native hosts at 32 NFE. Quality is accuracy/pass@1 (\%); bold marks
column-best quality. $\Delta$ gives gains over the host; overhead averages
the latency change across the four settings.}
\label{tab:fixed_budget}
\small
\begin{tabular}{lrrrrr}
\toprule
\rowcolor{TableHead}
& \multicolumn{2}{c}{LLaDA-8B} & \multicolumn{2}{c}{Dream-7B} & Overhead\\
\cmidrule(lr){2-3}\cmidrule(lr){4-5}
Variant & GSM8K & HumanEval & GSM8K & HumanEval & (\%)\\
\midrule
Host baseline & 66.4 & 30.5 & 70.2 & 45.7 & +0.0\\
Selection only & 70.2 & 35.4 & 74.5 & 49.4 & +1.7\\
Remask only & 67.8 & 32.3 & 71.4 & 47.6 & +1.3\\
\rowcolor{SetCommitTint}\textcolor{SetCommitBlue}{\textbf{Full \method{}}} & \textbf{71.4} & \textbf{37.2} & \textbf{76.0} & \textbf{50.6} & +2.4\\
\quad $\Delta$ vs. host (pp) & \gaincell{+5.0} & \gaincell{+6.7} & \gaincell{+5.8} & \gaincell{+4.9} & --\\
\midrule
Confidence + spacing & 67.3 & 32.9 & 72.4 & 47.0 & +1.2\\
Stratified random & 67.4 & 31.7 & 70.2 & 46.3 & +1.2\\
Shuffled risk & 67.1 & 31.1 & 70.4 & 45.7 & +1.7\\
\bottomrule
\end{tabular}
\end{table}

\paragraph{Selection and remask contributions.}
Figure~\ref{fig:mechanism} extends the component analysis to all four
tasks on both primary models. Selection contributes 3.5--5.4 pp over
the host, and remasking adds another 0.6--2.0 pp beyond selection-only.
Selection therefore accounts for most of the improvement, while revision
remains useful after preventive coordination. The interaction
$Q_{11}-Q_{10}-Q_{01}+Q_{00}$ ranges from $-1.0$ to $+0.3$ pp,
indicating partly overlapping contributions from preventive selection
and subsequent revision.
The same-state continuation analysis isolates the choice of a base
commitment set: source state, proposals, commit count, continuation
policy, and remaining budget are shared. Relative to the confidence-matched
control, the coordinated selector adds 1.86/0.88/0.07/0.64 pp at
12.5/25/50/75\% progress. The largest observed advantage occurs early,
when much context remains unresolved. This pattern is consistent with
preserving conditioning opportunities during the most uncertain part
of decoding.

\begin{figure}[!t]
\centering
\includegraphics[width=\linewidth]{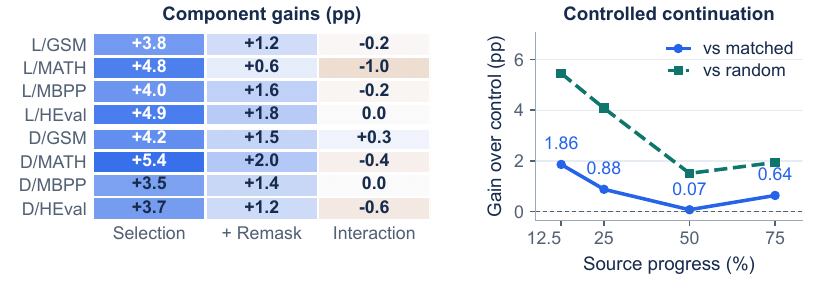}
\caption{\textbf{Component contributions and controlled continuation.}
Left: native hosts at 32 NFE; columns show $Q_{10}-Q_{00}$,
$Q_{11}-Q_{10}$, and $Q_{11}-Q_{10}-Q_{01}+Q_{00}$ (pp).
L/D denote LLaDA-8B/Dream-7B; blue/amber indicate positive/negative values.
Right: LLaDA-8B/GSM8K same-state selector gains over confidence-matched
and stratified-random controls. Labels show differences from the matched
control. All entries are point estimates.}
\label{fig:mechanism}
\end{figure}

\paragraph{Frozen configurations and deployment cost.}
Figure~\ref{fig:ablation} shows 4.9--7.4 pp gains with frozen settings
at 32 NFE, while retuning adds 0--1.06 pp (mean 0.39 pp). Most of the
reported improvement therefore survives without target-specific retuning,
supporting configuration reuse across the evaluated models and tasks.
On MATH-500, retuning adds zero in all four model columns, while the
frozen gains remain 5.4--7.4 pp. This provides a concrete example of
strong improvements from reusing the selected configuration.
Transfer uses model-specific decoding routes; selection boundaries
are detailed in Appendix~\ref{app:protocol}.
At the same NFE, added latency is 1.39--2.90\% on native hosts and
4.75--9.28\% on cached hosts. Caching reduces the host time against which
coordination is charged and introduces refresh work, making the relative
overhead larger. At the common quality targets in Table~\ref{tab:main}, fewer
evaluations offset this coordination cost and reduce end-to-end latency.

\begin{figure}[!htb]
\centering
\includegraphics[width=\linewidth]{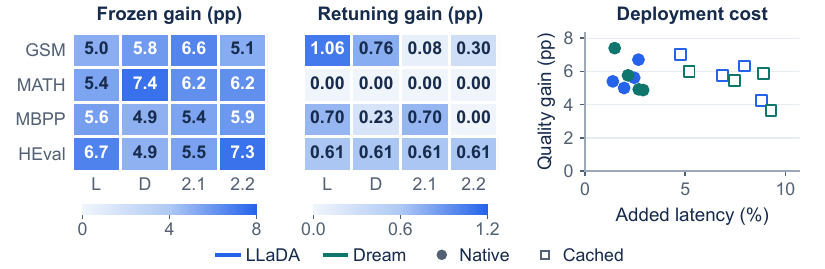}
\caption{\textbf{Frozen configurations and deployment cost at 32 NFE.}
Left: frozen quality gains over the corresponding host (pp).
Center: retuned minus frozen quality (pp); row order matches the left
matrix. L/D denote LLaDA-8B/Dream-7B; 2.1/2.2 denote additional LLaDA
targets. Color scales are separate and zero-anchored.
Right: each point pairs quality gain and added latency for one primary
model, task, and native/cached host; there are 16 points.
Colors identify models and markers identify hosts in this panel.}
\label{fig:ablation}
\end{figure}

\paragraph{Sensitivity and additional checks.}
\label{sec:runtime_results}
Additional analyses in Appendix~\ref{app:analysis} show that matched-budget
gains are largest at low NFE and diminish as the budget grows. Risk strength
and remask-cap effects characterize a nonmonotonic tradeoff between
coordination and token reliability. Incremental
speedups also persist under stricter quality margins: at a one-point
margin, geometric-mean speedups are $1.68\times$ for Native+\method{}
and $1.54\times$ for Fast fixed+\method{} across both primary models
and all four tasks. The same appendix provides parameter sweeps,
memory measurements, runtime breakdowns,
and the reporting protocol for these point estimates.

\section{Discussion}
\label{sec:discussion}

The experiments connect commitment structure to practical decoding
efficiency. The same-state comparisons show that changing which positions
are revealed can improve continuation quality, with the largest observed
advantage early in decoding. The component analysis further shows that
bounded remasking adds gains after coordinated selection. Together, these
results support preserving conditioning opportunities while retaining a
limited path to revise earlier commitments.
Baseline retention and the capped risk penalty keep this intervention
focused: most high-confidence host choices remain in the set, while the
entropy-weighted distance term delays uncertain neighboring predictions.
This design balances immediate progress with the context available in
later updates.

The host interface makes this decision complementary to caching and
scheduling. \method{} uses signals from the current forward pass and
preserves the net quota, while the improved quality allows fewer
evaluations at a common target. The resulting latency gains include
selection, exchange, and cache maintenance. This links the benefit of
coordinated commitment to end-to-end efficiency across the evaluated
hosts, with the strongest quality gains in the few-step regime.

\section{Conclusion}
\label{sec:conclusion}

We introduced \method{}, a training-free coordinator for parallel token commitment in diffusion language models. It combines reliability-aware set selection, which balances token reliability against local co-commitment risk, with bounded remasking, which revises vulnerable earlier commitments as context evolves, while preserving the host's per-step net commitment quota. On mathematical reasoning and code generation with LLaDA-8B and Dream-7B, it improves quality by $4.9$--$7.4$ percentage points at $32$ evaluations and achieves geometric-mean incremental speedups of $1.44$--$1.74\times$ at common quality targets across the evaluated model--host pairs. These results support coordinated commitment as a complementary component of efficient few-step decoding.

\clearpage

\subsection*{AI use statement}

In preparing this work, we used generative AI tools for language editing
and manuscript revision, and for feedback on methodological descriptions,
experimental reporting, and interpretation of results. This assistance
included consistency checks across the manuscript and the available
reproducibility documentation. AI-generated suggestions were assessed
against the manuscript, cited sources, and available artifacts. The authors
take full responsibility for the final text, claims, and accompanying artifacts.

\subsection*{Reproducibility statement}

Code and reproduction instructions are available in our public repository at
\url{https://github.com/Kissmetothemoon/SetCommit}.
The repository provides the method implementation,
policy configurations, environment requirements, LLaDA/Fast-dLLM paired
runners, and validation scripts. Section~\ref{sec:method} describes the method, while
Appendices~\ref{app:implementation} and~\ref{app:configuration} provide
implementation details and policy settings. Appendix~\ref{app:protocol}
describes the model and evaluation protocols, and Appendix~\ref{app:proof}
states the assumptions and proof for net-progress preservation.



\bibliography{iclr2027_conference}

@inproceedings{DBLP:conf/nips/NieZYZOHZLWL25,
  author       = {Shen Nie and
                  Fengqi Zhu and
                  Zebin You and
                  Xiaolu Zhang and
                  Jingyang Ou and
                  Jun Hu and
                  Jun Zhou and
                  Yankai Lin and
                  Ji{-}Rong Wen and
                  Chongxuan Li},
  editor       = {Danielle Belgrave and
                  Cheng Zhang and
                  Laura N. Montoya and
                  Hsuan{-}Tien Lin and
                  Razvan Pascanu and
                  Piotr Koniusz and
                  Marzyeh Ghassemi and
                  Nancy Chen and
                  Iv{\'{a}}n Vladimir Meza Ru{\'{\i}}z and
                  Arturo Loaiza{-}Bonilla},
  title        = {Large Language Diffusion Models},
  booktitle    = {Advances in Neural Information Processing Systems 38: Annual Conference
                  on Neural Information Processing Systems 2025, NeurIPS 2025, San Diego,
                  CA, USA, December 2-7, 2025 / Mexico City, Mexico, November 30 - December
                  5, 2025},
  year         = {2025},
  url          = {http://papers.nips.cc/paper\_files/paper/2025/hash/48b383b24230e0e6e649d9c98dae4d8c-Abstract-Conference.html},
  bibsource    = {dblp computer science bibliography, https://dblp.org}
}

@article{DBLP:journals/corr/abs-2508-15487,
  author       = {Jiacheng Ye and
                  Zhihui Xie and
                  Lin Zheng and
                  Jiahui Gao and
                  Zirui Wu and
                  Xin Jiang and
                  Zhenguo Li and
                  Lingpeng Kong},
  title        = {Dream 7B: Diffusion Large Language Models},
  journal      = {CoRR},
  volume       = {abs/2508.15487},
  year         = {2025},
  url          = {https://doi.org/10.48550/arXiv.2508.15487},
  doi          = {10.48550/ARXIV.2508.15487},
  eprinttype   = {arXiv},
  eprint       = {2508.15487},
  bibsource    = {dblp computer science bibliography, https://dblp.org}
}

@article{DBLP:journals/corr/abs-2602-06953,
  author       = {Lizhuo Luo and
                  Zhuoran Shi and
                  Jiajun Luo and
                  Zhi Wang and
                  Shen Ren and
                  Wenya Wang and
                  Tianwei Zhang},
  title        = {{DAWN:} Dependency-Aware Fast Inference for Diffusion LLMs},
  journal      = {CoRR},
  volume       = {abs/2602.06953},
  year         = {2026},
  url          = {https://doi.org/10.48550/arXiv.2602.06953},
  doi          = {10.48550/ARXIV.2602.06953},
  eprinttype   = {arXiv},
  eprint       = {2602.06953},
  bibsource    = {dblp computer science bibliography, https://dblp.org}
}

@article{DBLP:journals/corr/abs-2107-03374,
  author       = {Mark Chen and
                  Jerry Tworek and
                  Heewoo Jun and
                  Qiming Yuan and
                  Henrique Pond{\'{e}} de Oliveira Pinto and
                  Jared Kaplan and
                  Harri Edwards and
                  Yuri Burda and
                  Nicholas Joseph and
                  Greg Brockman and
                  Alex Ray and
                  Raul Puri and
                  Gretchen Krueger and
                  Michael Petrov and
                  Heidy Khlaaf and
                  Girish Sastry and
                  Pamela Mishkin and
                  Brooke Chan and
                  Scott Gray and
                  Nick Ryder and
                  Mikhail Pavlov and
                  Alethea Power and
                  Lukasz Kaiser and
                  Mohammad Bavarian and
                  Clemens Winter and
                  Philippe Tillet and
                  Felipe Petroski Such and
                  Dave Cummings and
                  Matthias Plappert and
                  Fotios Chantzis and
                  Elizabeth Barnes and
                  Ariel Herbert{-}Voss and
                  William Hebgen Guss and
                  Alex Nichol and
                  Alex Paino and
                  Nikolas Tezak and
                  Jie Tang and
                  Igor Babuschkin and
                  Suchir Balaji and
                  Shantanu Jain and
                  William Saunders and
                  Christopher Hesse and
                  Andrew N. Carr and
                  Jan Leike and
                  Joshua Achiam and
                  Vedant Misra and
                  Evan Morikawa and
                  Alec Radford and
                  Matthew Knight and
                  Miles Brundage and
                  Mira Murati and
                  Katie Mayer and
                  Peter Welinder and
                  Bob McGrew and
                  Dario Amodei and
                  Sam McCandlish and
                  Ilya Sutskever and
                  Wojciech Zaremba},
  title        = {Evaluating Large Language Models Trained on Code},
  journal      = {CoRR},
  volume       = {abs/2107.03374},
  year         = {2021},
  url          = {https://arxiv.org/abs/2107.03374},
  eprinttype   = {arXiv},
  eprint       = {2107.03374},
  bibsource    = {dblp computer science bibliography, https://dblp.org}
}

@inproceedings{DBLP:conf/iclr/LightmanKBEBLLS24,
  author       = {Hunter Lightman and
                  Vineet Kosaraju and
                  Yuri Burda and
                  Harrison Edwards and
                  Bowen Baker and
                  Teddy Lee and
                  Jan Leike and
                  John Schulman and
                  Ilya Sutskever and
                  Karl Cobbe},
  title        = {Let's Verify Step by Step},
  booktitle    = {The Twelfth International Conference on Learning Representations,
                  {ICLR} 2024, Vienna, Austria, May 7-11, 2024},
  publisher    = {OpenReview.net},
  year         = {2024},
  url          = {https://openreview.net/forum?id=v8L0pN6EOi},
  bibsource    = {dblp computer science bibliography, https://dblp.org}
}

@inproceedings{DBLP:conf/nips/SahooASGMCRK24,
  author       = {Subham S. Sahoo and
                  Marianne Arriola and
                  Yair Schiff and
                  Aaron Gokaslan and
                  Edgar Marroquin and
                  Justin T. Chiu and
                  Alexander M. Rush and
                  Volodymyr Kuleshov},
  editor       = {Amir Globersons and
                  Lester Mackey and
                  Danielle Belgrave and
                  Angela Fan and
                  Ulrich Paquet and
                  Jakub M. Tomczak and
                  Cheng Zhang},
  title        = {Simple and Effective Masked Diffusion Language Models},
  booktitle    = {Advances in Neural Information Processing Systems 37: Annual Conference
                  on Neural Information Processing Systems 2024, NeurIPS 2024, Vancouver,
                  BC, Canada, December 10 - 15, 2024},
  year         = {2024},
  url          = {http://papers.nips.cc/paper\_files/paper/2024/hash/eb0b13cc515724ab8015bc978fdde0ad-Abstract-Conference.html},
  bibsource    = {dblp computer science bibliography, https://dblp.org}
}

@inproceedings{DBLP:conf/nips/Ben-HamuGSNK25,
  author       = {Heli Ben{-}Hamu and
                  Itai Gat and
                  Daniel Severo and
                  Niklas Nolte and
                  Brian Karrer},
  editor       = {Danielle Belgrave and
                  Cheng Zhang and
                  Laura N. Montoya and
                  Hsuan{-}Tien Lin and
                  Razvan Pascanu and
                  Piotr Koniusz and
                  Marzyeh Ghassemi and
                  Nancy Chen and
                  Iv{\'{a}}n Vladimir Meza Ru{\'{\i}}z and
                  Arturo Loaiza{-}Bonilla},
  title        = {Accelerated Sampling from Masked Diffusion Models via Entropy Bounded
                  Unmasking},
  booktitle    = {Advances in Neural Information Processing Systems 38: Annual Conference
                  on Neural Information Processing Systems 2025, NeurIPS 2025, San Diego,
                  CA, USA, December 2-7, 2025 / Mexico City, Mexico, November 30 - December
                  5, 2025},
  year         = {2025},
  url          = {http://papers.nips.cc/paper\_files/paper/2025/hash/510e22f4a2da5212a64f1591736e2eaf-Abstract-Conference.html},
  bibsource    = {dblp computer science bibliography, https://dblp.org}
}

@inproceedings{DBLP:conf/acl/ZhouHH26,
  author       = {Xueyu Zhou and
                  Yangrong Hu and
                  Jian Huang},
  editor       = {Maria Liakata and
                  Viviane P. Moreira and
                  Jiajun Zhang and
                  David Jurgens},
  title        = {{DOS:} Dependency-Oriented Sampler for Masked Diffusion Language Models},
  booktitle    = {Findings of the Association for Computational Linguistics, {ACL} 2026,
                  San Diego, California, United States, July 2-7, 2026},
  pages        = {17404--17419},
  publisher    = {Association for Computational Linguistics},
  year         = {2026},
  url          = {https://doi.org/10.18653/v1/2026.findings-acl.861},
  doi          = {10.18653/V1/2026.FINDINGS-ACL.861},
  bibsource    = {dblp computer science bibliography, https://dblp.org}
}

@article{DBLP:journals/corr/abs-2110-14168,
  author       = {Karl Cobbe and
                  Vineet Kosaraju and
                  Mohammad Bavarian and
                  Mark Chen and
                  Heewoo Jun and
                  Lukasz Kaiser and
                  Matthias Plappert and
                  Jerry Tworek and
                  Jacob Hilton and
                  Reiichiro Nakano and
                  Christopher Hesse and
                  John Schulman},
  title        = {Training Verifiers to Solve Math Word Problems},
  journal      = {CoRR},
  volume       = {abs/2110.14168},
  year         = {2021},
  url          = {https://arxiv.org/abs/2110.14168},
  eprinttype   = {arXiv},
  eprint       = {2110.14168},
  bibsource    = {dblp computer science bibliography, https://dblp.org}
}

@inproceedings{
wu2026fastdllm,
title={Fast-d{LLM}: Training-free Acceleration of Diffusion {LLM} by Enabling {KV} Cache and Parallel Decoding},
author={Chengyue Wu and Hao Zhang and Shuchen Xue and Zhijian Liu and Shizhe Diao and Ligeng Zhu and Ping Luo and Song Han and Enze Xie},
booktitle={The Fourteenth International Conference on Learning Representations},
year={2026},
url={https://openreview.net/forum?id=3Z3Is6hnOT}
}

@article{DBLP:journals/corr/abs-2108-07732,
  author       = {Jacob Austin and
                  Augustus Odena and
                  Maxwell I. Nye and
                  Maarten Bosma and
                  Henryk Michalewski and
                  David Dohan and
                  Ellen Jiang and
                  Carrie J. Cai and
                  Michael Terry and
                  Quoc V. Le and
                  Charles Sutton},
  title        = {Program Synthesis with Large Language Models},
  journal      = {CoRR},
  volume       = {abs/2108.07732},
  year         = {2021},
  url          = {https://arxiv.org/abs/2108.07732},
  eprinttype   = {arXiv},
  eprint       = {2108.07732},
  bibsource    = {dblp computer science bibliography, https://dblp.org}
}

@inproceedings{DBLP:conf/nips/MaYFW25,
  author       = {Xinyin Ma and
                  Runpeng Yu and
                  Gongfan Fang and
                  Xinchao Wang},
  editor       = {Danielle Belgrave and
                  Cheng Zhang and
                  Laura N. Montoya and
                  Hsuan{-}Tien Lin and
                  Razvan Pascanu and
                  Piotr Koniusz and
                  Marzyeh Ghassemi and
                  Nancy Chen and
                  Iv{\'{a}}n Vladimir Meza Ru{\'{\i}}z and
                  Arturo Loaiza{-}Bonilla},
  title        = {dKV-Cache: The Cache for Diffusion Language Models},
  booktitle    = {Advances in Neural Information Processing Systems 38: Annual Conference
                  on Neural Information Processing Systems 2025, NeurIPS 2025, San Diego,
                  CA, USA, December 2-7, 2025 / Mexico City, Mexico, November 30 - December
                  5, 2025},
  year         = {2025},
  url          = {http://papers.nips.cc/paper\_files/paper/2025/hash/db0ee27cb50dd9087b133f6e7d28a90e-Abstract-Conference.html},
  bibsource    = {dblp computer science bibliography, https://dblp.org}
}

@inproceedings{DBLP:conf/aaai/SongLLLHGHQ26,
  author       = {Yuerong Song and
                  Xiaoran Liu and
                  Ruixiao Li and
                  Zhigeng Liu and
                  Zengfeng Huang and
                  Qipeng Guo and
                  Ziwei He and
                  Xipeng Qiu},
  editor       = {Sven Koenig and
                  Chad Jenkins and
                  Matthew E. Taylor},
  title        = {Sparse-dLLM: Accelerating Diffusion LLMs with Dynamic Cache Eviction},
  booktitle    = {Fortieth {AAAI} Conference on Artificial Intelligence, Thirty-Eighth
                  Conference on Innovative Applications of Artificial Intelligence,
                  Sixteenth Symposium on Educational Advances in Artificial Intelligence,
                  {AAAI} 2026, Singapore, January 20-27, 2026},
  pages        = {33038--33046},
  publisher    = {{AAAI} Press},
  year         = {2026},
  url          = {https://doi.org/10.1609/aaai.v40i39.40586},
  doi          = {10.1609/AAAI.V40I39.40586},
  bibsource    = {dblp computer science bibliography, https://dblp.org}
}

@inproceedings{DBLP:conf/nips/IsraelBG25,
  author       = {Daniel Israel and
                  Guy Van den Broeck and
                  Aditya Grover},
  editor       = {Danielle Belgrave and
                  Cheng Zhang and
                  Laura N. Montoya and
                  Hsuan{-}Tien Lin and
                  Razvan Pascanu and
                  Piotr Koniusz and
                  Marzyeh Ghassemi and
                  Nancy Chen and
                  Iv{\'{a}}n Vladimir Meza Ru{\'{\i}}z and
                  Arturo Loaiza{-}Bonilla},
  title        = {Accelerating Diffusion LLMs via Adaptive Parallel Decoding},
  booktitle    = {Advances in Neural Information Processing Systems 38: Annual Conference
                  on Neural Information Processing Systems 2025, NeurIPS 2025, San Diego,
                  CA, USA, December 2-7, 2025 / Mexico City, Mexico, November 30 - December
                  5, 2025},
  year         = {2025},
  url          = {http://papers.nips.cc/paper\_files/paper/2025/hash/4c7d31b2da05a17f15101db5b37a1e24-Abstract-Conference.html},
  bibsource    = {dblp computer science bibliography, https://dblp.org}
}

@inproceedings{DBLP:conf/iclr/WeiZLZWQLZ26,
  author       = {Qingyan Wei and
                  Yaojie Zhang and
                  Zhiyuan Liu and
                  Puyu Zeng and
                  Yuxuan Wang and
                  Biqing Qi and
                  Dongrui Liu and
                  Linfeng Zhang},
  editor       = {Carl Vondrick and
                  Bharath Hariharan and
                  Colin Raffel and
                  Lerrel Pinto and
                  Diyi Yang and
                  Aleksandra Faust},
  title        = {Accelerating Diffusion Large Language Models with SlowFast Sampling:
                  The Three Golden Principles},
  booktitle    = {The Fourteenth International Conference on Learning Representations,
                  {ICLR} 2026, Rio de Janeiro, Brazil, April 23-27, 2026},
  publisher    = {proceedings.iclr.cc},
  year         = {2026},
  url          = {https://proceedings.iclr.cc/paper\_files/paper/2026/hash/08487598819cba9feca884ef0d442950-Abstract-Conference.html},
  bibsource    = {dblp computer science bibliography, https://dblp.org}
}

@inproceedings{DBLP:conf/nips/AustinJHTB21,
  author       = {Jacob Austin and
                  Daniel D. Johnson and
                  Jonathan Ho and
                  Daniel Tarlow and
                  Rianne van den Berg},
  editor       = {Marc'Aurelio Ranzato and
                  Alina Beygelzimer and
                  Yann N. Dauphin and
                  Percy Liang and
                  Jennifer Wortman Vaughan},
  title        = {Structured Denoising Diffusion Models in Discrete State-Spaces},
  booktitle    = {Advances in Neural Information Processing Systems 34: Annual Conference
                  on Neural Information Processing Systems 2021, NeurIPS 2021, December
                  6-14, 2021, virtual},
  pages        = {17981--17993},
  year         = {2021},
  url          = {https://proceedings.neurips.cc/paper/2021/hash/958c530554f78bcd8e97125b70e6973d-Abstract.html},
  bibsource    = {dblp computer science bibliography, https://dblp.org}
}

@inproceedings{DBLP:conf/iclr/LiuBNB25,
  author       = {Anji Liu and
                  Oliver Broadrick and
                  Mathias Niepert and
                  Guy Van den Broeck},
  title        = {Discrete Copula Diffusion},
  booktitle    = {The Thirteenth International Conference on Learning Representations,
                  {ICLR} 2025, Singapore, April 24-28, 2025},
  publisher    = {OpenReview.net},
  year         = {2025},
  url          = {https://openreview.net/forum?id=FXw0okNcOb},
  bibsource    = {dblp computer science bibliography, https://dblp.org}
}

@inproceedings{DBLP:conf/iclr/XuGKNXLEV25,
  author       = {Minkai Xu and
                  Tomas Geffner and
                  Karsten Kreis and
                  Weili Nie and
                  Yilun Xu and
                  Jure Leskovec and
                  Stefano Ermon and
                  Arash Vahdat},
  title        = {Energy-Based Diffusion Language Models for Text Generation},
  booktitle    = {The Thirteenth International Conference on Learning Representations,
                  {ICLR} 2025, Singapore, April 24-28, 2025},
  publisher    = {OpenReview.net},
  year         = {2025},
  url          = {https://openreview.net/forum?id=sL2F9YCMXf},
  bibsource    = {dblp computer science bibliography, https://dblp.org}
}

@inproceedings{DBLP:conf/nips/HendrycksBKABTS21,
  author       = {Dan Hendrycks and
                  Collin Burns and
                  Saurav Kadavath and
                  Akul Arora and
                  Steven Basart and
                  Eric Tang and
                  Dawn Song and
                  Jacob Steinhardt},
  editor       = {Joaquin Vanschoren and
                  Sai{-}Kit Yeung},
  title        = {Measuring Mathematical Problem Solving With the {MATH} Dataset},
  booktitle    = {Proceedings of the Neural Information Processing Systems Track on
                  Datasets and Benchmarks 1, NeurIPS Datasets and Benchmarks 2021, December
                  2021, virtual},
  year         = {2021},
  url          = {https://datasets-benchmarks-proceedings.neurips.cc/paper/2021/hash/be83ab3ecd0db773eb2dc1b0a17836a1-Abstract-round2.html},
  bibsource    = {dblp computer science bibliography, https://dblp.org}
}

@inproceedings{
liu2026dllmcache,
title={d{LLM}-Cache: Accelerating Diffusion Large Language Models with Adaptive Caching},
author={Zhiyuan Liu and Yicun Yang and Yaojie Zhang and Junjie Chen and Chang Zou and Qingyan Wei and Shaobo Wang and Yichen Zhu and Linfeng Zhang},
booktitle={Forty-third International Conference on Machine Learning},
year={2026},
url={https://openreview.net/forum?id=guLrgDb2jB}
}

@inproceedings{
li2026diffusion,
title={Diffusion Language Model Knows the Answer Before It Decodes},
author={Pengxiang Li and Yefan Zhou and Dilxat Muhtar and Lu Yin and Shilin Yan and Li Shen and Yi Liang and Soroush Vosoughi and Shiwei Liu},
booktitle={The Fourteenth International Conference on Learning Representations},
year={2026},
url={https://openreview.net/forum?id=g88nt4ieTG}
}

@inproceedings{DBLP:conf/nips/WangSSK25,
  author       = {Guanghan Wang and
                  Yair Schiff and
                  Subham S. Sahoo and
                  Volodymyr Kuleshov},
  editor       = {Danielle Belgrave and
                  Cheng Zhang and
                  Laura N. Montoya and
                  Hsuan{-}Tien Lin and
                  Razvan Pascanu and
                  Piotr Koniusz and
                  Marzyeh Ghassemi and
                  Nancy Chen and
                  Iv{\'{a}}n Vladimir Meza Ru{\'{\i}}z and
                  Arturo Loaiza{-}Bonilla},
  title        = {Remasking Discrete Diffusion Models with Inference-Time Scaling},
  booktitle    = {Advances in Neural Information Processing Systems 38: Annual Conference
                  on Neural Information Processing Systems 2025, NeurIPS 2025, San Diego,
                  CA, USA, December 2-7, 2025 / Mexico City, Mexico, November 30 - December
                  5, 2025},
  year         = {2025},
  url          = {http://papers.nips.cc/paper\_files/paper/2025/hash/d877ea0dd78bdbe54830670618c1de09-Abstract-Conference.html},
  bibsource    = {dblp computer science bibliography, https://dblp.org}
}

@inproceedings{DBLP:conf/icml/RutteFDOS025,
  author       = {Dimitri von R{\"{u}}tte and
                  Janis Fluri and
                  Yuhui Ding and
                  Antonio Orvieto and
                  Bernhard Sch{\"{o}}lkopf and
                  Thomas Hofmann},
  editor       = {Aarti Singh and
                  Maryam Fazel and
                  Daniel Hsu and
                  Simon Lacoste{-}Julien and
                  Felix Berkenkamp and
                  Tegan Maharaj and
                  Kiri Wagstaff and
                  Jerry Zhu},
  title        = {Generalized Interpolating Discrete Diffusion},
  booktitle    = {Forty-second International Conference on Machine Learning, {ICML}
                  2025, Vancouver, BC, Canada, July 13-19, 2025},
  series       = {Proceedings of Machine Learning Research},
  volume       = {267},
  publisher    = {{PMLR} / OpenReview.net},
  year         = {2025},
  url          = {https://proceedings.mlr.press/v267/von-rutte25a.html},
  bibsource    = {dblp computer science bibliography, https://dblp.org}
}

@inproceedings{DBLP:conf/nips/ZhaoSCDML25,
  author       = {Yixiu Zhao and
                  Jiaxin Shi and
                  Feng Chen and
                  Shaul Druckmann and
                  Lester Mackey and
                  Scott W. Linderman},
  editor       = {Danielle Belgrave and
                  Cheng Zhang and
                  Laura N. Montoya and
                  Hsuan{-}Tien Lin and
                  Razvan Pascanu and
                  Piotr Koniusz and
                  Marzyeh Ghassemi and
                  Nancy Chen and
                  Iv{\'{a}}n Vladimir Meza Ru{\'{\i}}z and
                  Arturo Loaiza{-}Bonilla},
  title        = {Informed Correctors for Discrete Diffusion Models},
  booktitle    = {Advances in Neural Information Processing Systems 38: Annual Conference
                  on Neural Information Processing Systems 2025, NeurIPS 2025, San Diego,
                  CA, USA, December 2-7, 2025 / Mexico City, Mexico, November 30 - December
                  5, 2025},
  year         = {2025},
  url          = {http://papers.nips.cc/paper\_files/paper/2025/hash/b67d9d07faf071532c420ed5b140756b-Abstract-Conference.html},
  bibsource    = {dblp computer science bibliography, https://dblp.org}
}

@inproceedings{DBLP:conf/iclr/DeschenauxG25,
  author       = {Justin Deschenaux and
                  Caglar Gulcehre},
  title        = {Beyond Autoregression: Fast LLMs via Self-Distillation Through Time},
  booktitle    = {The Thirteenth International Conference on Learning Representations,
                  {ICLR} 2025, Singapore, April 24-28, 2025},
  publisher    = {OpenReview.net},
  year         = {2025},
  url          = {https://openreview.net/forum?id=uZ5K4HeNwd},
  bibsource    = {dblp computer science bibliography, https://dblp.org}
}
\bibliographystyle{iclr2027_conference}

\clearpage
\appendix
\raggedbottom
\setlength{\floatsep}{12pt plus 2pt minus 2pt}
\setlength{\textfloatsep}{14pt plus 2pt minus 2pt}
\setlength{\intextsep}{12pt plus 2pt minus 2pt}
\clubpenalty=10000
\widowpenalty=10000

\section{Preliminary Commitment Analyses}
\label{app:preliminary}

Figures~\ref{fig:low_step} and~\ref{fig:commitment_evidence} present the
main preliminary evidence; Figure~\ref{fig:commitment_diagnostics}
provides the accompanying distribution and prompt-outcome diagnostics.
\subsection{Source States and Local-Response Probes}
Figure~\ref{fig:low_step} (left) uses the top-$k$ positions from the same source
state as Figure~\ref{fig:commitment_evidence} (left).
The 16 selected positions lie in 561--583 and 613--632; the histogram
uses 10-position bins and retains every selection in the 520--680 view.
Its vertical axis reports commitment counts, not posterior probabilities,
which are absent from the compact source data.

The local-response probe freezes a source state and reveals a predicted
token at a signed distance in
$\{\pm1,\pm2,\pm4,\pm8,\pm16,\pm32\}$.
Figure~\ref{fig:low_step} (right) averages the absolute change in the
masked target's predicted-token posterior, pooling probes by absolute
distance. The response is in probability units, not accuracy points or
Jensen--Shannon divergence. Distances 1/2/4/8/16/32 contain
3380/3071/2881/2486/2037/1977 probe observations, respectively.
The displayed SE is the sample standard deviation divided by the square
root of the probe count. Its mean $\pm1.96$ SE band is a probe-level
normal approximation, not a prompt-clustered confidence interval.

\subsection{Paired Continuation Design}
The paired GSM8K analysis uses 32 steps and 240 prompts. Source progress
is the fraction of the source denoising trajectory reached before the
intervention, not the fraction of tokens selected by the alternative.
The anchors $0.125$, $0.25$, $0.5$, and $0.75$ yield 960 matched states
per rule. Branches share the source state, commit count, token proposals,
continuation policy, and remaining evaluation budget.
The center panel averages each prompt's paired accuracy differences over
anchors before averaging over prompts; its 95\% bootstrap intervals
cluster all anchors of a prompt together. The right panel uses 10,000
paired prompt bootstrap resamples separately at each anchor and reports
the 2.5th/97.5th percentiles. These are pointwise intervals, not a
simultaneous band across progress.

\begin{figure}[!htb]
    \centering
    \includegraphics[width=\linewidth]{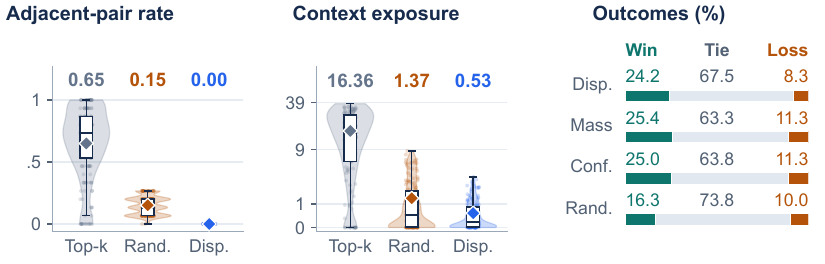}
    \caption{\textbf{Commitment distributions and paired outcome balance.}
    Left and center: adjacent-pair rate and mutual-context exposure across 960
    matched states per rule. Diamonds and values denote means;
    exposure uses $\log_{10}(1+x)$ spacing with raw-unit ticks.
    Right: win/tie/loss shares for 240 GSM8K prompts after averaging each
    prompt's paired score differences across progress anchors.
    Rows follow Figure~\ref{fig:commitment_evidence} (center).}
    \label{fig:commitment_diagnostics}
\end{figure}

\paragraph{Commitment structure across matched states.}
Confidence top-$k$, random, and dispersed selection produce mean
adjacent-pair rates of 0.65, 0.15, and approximately zero, respectively
(Figure~\ref{fig:commitment_diagnostics} (left)).
Their mean local context-exposure diagnostics are 16.36, 1.37, and 0.53
(Figure~\ref{fig:commitment_diagnostics} (center)).
Dispersed selection reduces exposure by 97\% relative to top-$k$,
consistent with leaving intervening positions available to condition
on the newly revealed context.

\subsection{Metrics and Aggregation}
The compact reproduction bundle supplies the selected positions for
Figure~\ref{fig:low_step} (left) and the points and standard errors for
Figure~\ref{fig:low_step} (center and right), 960 matched-state observations per rule
for Figure~\ref{fig:commitment_diagnostics} (left and center), and the effect estimates,
intervals, and prompt-outcome shares for
Figures~\ref{fig:commitment_evidence} (center and right) and~\ref{fig:commitment_diagnostics} (right).
For Figure~\ref{fig:low_step} (center), let $e_{p,t}$ be the recorded
excess-adjacent-pair rate and $b_{p,t}$ the commitment batch size for
prompt $p$ at step $t$. The plotted quantity is
\[
 E=\frac{1}{n}\sum_{p=1}^{n} E_p,\qquad
 E_p=\sum_t e_{p,t}(b_{p,t}-1),\qquad
 \mathrm{SE}(E)=\frac{\mathrm{SD}(E_1,\ldots,E_n)}{\sqrt n}.
\]
Each Discovery/Holdout summary has $n=240$ prompts. Shading shows
$E\pm1.96\,\mathrm{SE}(E)$ with prompts as the aggregation unit.
The compact summaries do not retain the state-level construction of
the excess-rate reference; this panel is therefore a descriptive
clustering diagnostic, with the controlled continuation comparison
reported separately in Figure~\ref{fig:commitment_evidence}.
The exposure panel uses $\log_{10}(1+x)$ for display, with raw-unit tick
labels; its mean marker is placed at the transform of the raw mean.
In Figure~\ref{fig:commitment_diagnostics} (left and center), boxes show the interquartile range and median,
with whiskers extending to observations within 1.5 interquartile ranges;
diamonds and printed values denote arithmetic means.
Distributions use all 960 states per rule, while dots show a deterministic
subset with fixed jitter for readability.
Figure~\ref{fig:commitment_evidence} (left) zooms to positions 400--950 and retains all 16
commitments per rule from the original example.
All displayed intervals are retained from the supplied summaries.

For the outcome shares, let $d_{p,a}$ be the paired score difference for
prompt $p$ at source-progress anchor $a$.
The script averages these differences within each prompt and classifies
the sign of $\bar d_p$ as win, tie, or loss.
Consequently, Figure~\ref{fig:commitment_diagnostics} (right) summarizes the
frequency of positive and negative prompt-level effects, while
Figure~\ref{fig:commitment_evidence} (center) reports their mean magnitude.
The outcome panel presents all twelve shares as 100\% stacked horizontal bars.
Segment widths retain full precision, with percentage labels rounded to
one decimal place above each bar. Its rows follow Figure~\ref{fig:commitment_evidence} (center): dispersed, posterior-mass-matched,
confidence-stratified, and random controls.

\subsection{Effect Estimates and Interpretation}
The preliminary interventions make commitment organization visible
through dispersed, posterior-mass-matched, confidence-stratified, and
random controls.
The dispersed effect is 4.27 percentage points (95\% CI: $[1.25,7.19]$).
Posterior-mass matching controls selected confidence mass, while
confidence stratification controls the confidence profile; their
effects are 3.33 pp ($[0.00,6.56]$) and 3.54 pp ($[0.21,6.67]$).
Unrestricted random selection yields 1.04 pp ($[-2.19,4.17]$).
For dispersed selection at 12.5/25/50/75\% progress, the effects are
10.42/3.75/$-0.42$/3.33 pp, with intervals
$[5.00,15.83]$, $[-0.83,8.33]$, $[-4.17,3.33]$, and $[0.42,6.25]$.
The largest point estimate occurs early; the mid-trajectory intervals
include zero, and the late improvement makes the trend nonmonotonic.
Comparing these pointwise intervals does not establish a significant
difference between progress anchors.
The full-module experiments in Section~\ref{sec:mechanism_results}
extend this analysis to the finalized coordinated selector and its
combination with exchange.

\FloatBarrier
\Needspace{22\baselineskip}
\section{Net-Progress Accounting}
\label{app:proof}

For the remask exchange in Section~\ref{sec:integration}, let $A_t$ denote fresh
commitments and $R_t$ previously resolved positions remasked at step $t$.
The mask update and net-quota contract are
\begin{equation}
    M_{t+1}=(M_t\setminus A_t)\cup R_t,
    \qquad
    |M_t|-|M_{t+1}|=|A_t|-|R_t|.
    \label{eq:net_progress}
\end{equation}
\begin{equation}
    |A_t|-|R_t|=q_t,
    \label{eq:quota_contract}
\end{equation}

\begin{proposition}[Net-progress preservation]
\label{prop:progress}
Suppose the host and module satisfy
Equations~\eqref{eq:net_progress}--\eqref{eq:quota_contract} at every
executed step.
Then, after \(s\) steps,
\[
    |M_s|=|M_0|-\sum_{t=0}^{s-1}q_t.
\]
In particular, a feasible schedule with
\(\sum_{t=0}^{T-1}q_t=|M_0|\) resolves the initial mask set in \(T\) steps,
assuming no earlier task stopping and no additional positions introduced.
\end{proposition}

\begin{proof}
Let \(A_t=S_t\cup E_t\subseteq M_t\) contain fresh commitments and
\(R_t\subseteq M_t^c\) contain previously resolved positions remasked by the
accepted update, with \(S_t\cap E_t=\varnothing\), \(|S_t|=q_t\), and \(|E_t|=|R_t|\).
The sets \(M_t\setminus A_t\) and \(R_t\) are disjoint, giving
\[
 |M_{t+1}|=|M_t|-|A_t|+|R_t|.
\]
Under the fixed-budget contract \(|A_t|-|R_t|=q_t\), this becomes
\(|M_{t+1}|=|M_t|-q_t\).
Telescoping over the executed steps proves
Proposition~\ref{prop:progress}.
Overwrites leave mask status unchanged and do not appear in this identity.
\end{proof}

The conditions are local and directly testable.
A trace records fresh reveals, remasks, posterior gains, active
block, quota, and remaining masks.
An admissible exchange must respect edit eligibility and the fresh positions
available to satisfy the quota.
Early task stopping is recorded separately from completion of a fixed-length
mask schedule.
For blockwise generation, an explicit ledger records positions entering
each active block.

\FloatBarrier
\Needspace{8\baselineskip}
\section{Implementation Details}
\label{app:implementation}

The Native/Fast implementation uses policy
\texttt{dacs\_v5\_exchange\_020}. Its details follow before the separate
EB adapter in Appendix~\ref{app:eb}. Sets are restricted to the current
host-editable block; positions retain their sequence indices.

\subsection{Native and Fast Policies}
\paragraph{Hyperparameters and tolerances.}
For the Native/Fast configuration, the shortlist multiplier is
$\kappa=4$, baseline retention is $\rho=0.75$, and the penalty cap is
$\beta=0.10$. Exchange uses protection rank $K=2$, ratio $\alpha=0.20$,
and progress window $[s_{\min},s_{\max}]=[0,0.8]$.
The posterior-gain margin $\epsilon_1=0.002$ applies to both withdrawal
eligibility and each paired replacement, matching
\texttt{min\_posterior\_gain}. The risk-comparison tolerance is
$\epsilon_2=10^{-7}$; it accommodates floating-point roundoff.
Appendix~\ref{app:configuration} lists the corresponding configuration
fields and the distinct EB settings.

\Needspace{13\baselineskip}
\paragraph{Signals and candidate construction.}
At a single forward pass, compute $y_i=\arg\max_v z_i(v)$,
$c_i=p_i(y_i)$, $o_i=p_i(x_i)$, and margin $m_i=z_{i,(1)}-z_{i,(2)}$.
Normalized entropy is
\[
u_i=-\frac{\sum_{v\in\mathcal V}p_i(v)\log p_i(v)}{\log|\mathcal V|}.
\]
Margins are normalized over the masked candidates:
\[
\bar m_i=\frac{m_i-\min_{j\in M_t}m_j}
{\max_{j\in M_t}m_j-\min_{j\in M_t}m_j}.
\]
If the denominator is below $10^{-8}$, all finite candidates receive
score one. The posterior baseline $B_t$ is top-$q_t$ by $c_i$, with
ineligible positions masked out and the underlying top-$k$ tie behavior
retained. Margin ordering uses position index to break ties.
The first $\min(|M_t|,\lceil4q_t\rceil)$ margin-ranked positions form a
shortlist, augmented with all of $B_t$.
During greedy construction, remaining slots are reserved for baseline
positions whenever necessary to meet $|S_t\cap B_t|\ge\lceil0.75q_t\rceil$.
Soft filtering is relaxed if needed to fill the quota. Invalid host
quotas are recorded as failed runs.

\paragraph{Progress and set risk.}
Let $s_t=(t+1)/T$ denote progress supplied by the host's decoding loop,
using its configured progress scale $T$. For adaptive hosts this scale
is distinct from the realized NFE reported after generation.
The schedules are $\lambda_t=0.03-0.02s_t$ and $\tau_t=16-12s_t$.
For this recorded policy, $d_t=|M_t|/L$ is the mask density and
the implementation-specific risk kernel is
\begin{equation}
    a_{ij}^{(t)}=d_t\exp(-|i-j|/\tau_t)u_i^{(t)}u_j^{(t)}
    \quad(i\ne j),\qquad a_{ii}^{(t)}=0.
    \label{eq:local_risk}
\end{equation}
The capped greedy selection score is
\begin{equation}
    g_t(i\mid S)=\bar m_i^{(t)}-
    \min\!\left\{\lambda_t\sum_{j\in S}a_{ij}^{(t)},\,0.10\right\}.
    \label{eq:selection_score}
\end{equation}
With the mask density and current entropy held fixed, set risk is
\begin{equation}
H_t(U)=\frac12\sum_{i,j\in U}a_{ij}^{(t)}.
\label{eq:set_risk}
\end{equation}
The kernel is a locality-based proxy, not a joint-dependency estimator.
Selection uses the capped incremental penalty in
Equation~\eqref{eq:selection_score}; there is no hard selection-risk budget.

\paragraph{Exchange acceptance.}
Let $C_t$ be the visible positions in the editable block.
Eligible withdrawals satisfy $x_i\ne y_i$,
$x_i\notin\operatorname{Top2}(z_i)$, and $c_i>o_i+0.002$;
they are ordered by increasing $o_i$.
The routine tries exchange sizes from the cap down to one, constructs
$E_t$ within $M_t\setminus S_t$ using a posterior baseline and the same
greedy selection rule, and pairs sorted extra-token posteriors with
sorted withdrawn-token posteriors. Every pair must exceed the
$0.002$ gain threshold, and
\[
H_t\bigl((C_t\setminus R_t)\cup S_t\cup E_t\bigr)
\le H_t(C_t\cup S_t)+10^{-7}.
\]
The first passing size is accepted; otherwise $R_t=E_t=\varnothing$.
All probabilities are evaluated under the current forward, including
predictions at still-visible old positions. These are local screening
conditions; final answer quality is assessed by the task scorer.

\paragraph{Native and Fast integration.}
The uncached LLaDA fixed-budget path operates on 32-token blocks with
uniformly assigned quotas. Fast-dLLM v1 uses the corresponding LLaDA/Dream
dual-cache paths. In threshold mode, confidence at least $0.9$, together
with forced acceptance of a highest-confidence position, determines $q_t$.
The module preserves this count; reordered positions need not individually
satisfy the host's original confidence threshold. Future counts and total
NFE can change as the trajectory changes.
The first forward establishes the cache; subsequent forwards refresh
the active block through the host's replacement-position mask.
Fast-Dream retains its native first-token prefill action before later
block updates enter the selector. Fixed-budget runs do not terminate
early on EOS; adaptive blocks terminate when no masks remain.
Validation checks disjointness, quota preservation, exchange gain, set
risk, and final mask completion. Violations are recorded as failed or rejected runs.

\subsection{EB-Sampler Integration}
\label{app:eb}

\paragraph{Implementation and scope.}
Our EB reimplementation follows \citet{DBLP:conf/nips/Ben-HamuGSNK25},
arXiv v1, and the archived author source. The release manifest records
source checksums and revisions.
The audited LLaDA adapter uses \texttt{GSAI-ML/LLaDA-8B-Instruct}, BF16,
batch size one, and a 512-position generation block. The documented E5
runs cover GSM8K-COT and MBPP. Dream results use a separate runner for
\texttt{Dream-org/Dream-v0-Instruct-7B} with 256 generated positions.
The adapter validation described here covers the LLaDA route.

\paragraph{Entropy-based quota.}
Let $h_i$ denote the entropy computed by the EB scoring path and let
$\pi$ order the remaining masked positions by increasing error proxy.
LLaDA uses one minus confidence; the separate Dream policies use entropy.
EB chooses the longest prefix satisfying
\begin{equation}
q_t=\max\left\{k\in\{1,\ldots,|M_t|\}:\;
\sum_{j=1}^{k}h_{\pi_j}-\max_{1\le j\le k}h_{\pi_j}\le\gamma\right\},
\qquad B_t^{\mathrm{EB}}=\{\pi_1,\ldots,\pi_{q_t}\}.
\label{eq:eb_quota}
\end{equation}
For the nonnegative thresholds used here, the single-position prefix is
always admissible. At $\gamma=0$ this often yields one commitment, but
zero-entropy candidates can permit more. The scan pairs EB and
EB+\method{} at four recorded $\gamma$ settings. In the LLaDA adapter,
token proposals are the raw-logit argmax; EOS suppression affects only
the confidence/entropy scoring path, not the generated-token argmax.

\paragraph{Module contract and policy.}
After the forward, EB computes $q_t$ and $B_t^{\mathrm{EB}}$; the module
uses this prefix as its baseline, selects $q_t$ base positions, and applies
remask-and-recommit exchange. The runtime enforces
$|A_t|-|R_t|=q_t$. No additional model forward is introduced, but the
changed trajectory can change later quotas and total NFE.
The coordinated set can exceed the entropy bound in
Equation~\eqref{eq:eb_quota}; the integration preserves EB's quota.
Both the original and coordinated-set entropy bounds are logged.
The documented EB runs use \texttt{dacs\_v5\_quality\_guard}, with
posterior reliability, 90\% baseline retention, a $2q_t$ shortlist,
and structural protection (Table~\ref{tab:configuration}).
Exchange is capped at $\lfloor0.10q_t\rfloor$, subject to fresh-position
availability, progress in $[0.10,0.60]$, changed predictions, top-2
exclusion, posterior gain above $0.005$, and nonincreasing set risk.
Native EB disables exchange. A nonzero configured ratio is not an
observed exchange rate; at $q_t=1$ the cap is zero.

\paragraph{Stopping, cost, and validation.}
LLaDA EB continues until the full generation region is unmasked.
Task stop strings are applied afterward: actual NFE counts all forwards
to completion, whereas effective NFE identifies the trace point at which
the stop condition and its preceding tokens have been resolved.
The nominal \texttt{steps=512} field accounts for generation length,
not a fixed EB step budget. Effective NFE must not be substituted for
actual computation or measured latency.
The separate Dream runner extends its nominal step allowance when masks
remain.
Both routes recompute the full sequence without a model KV cache, so
remasked positions are naturally refreshed at the next regular forward.
A step-local selector cache only avoids repeated set construction.
The LLaDA audit checks valid commit sets, per-step and total net progress,
risk nonincrease, trace lengths, NFE records, and final mask completion;
invalid outputs are recorded as rejected runs.

\FloatBarrier
\Needspace{34\baselineskip}
\section{Experimental Policy Configuration}
\label{app:configuration}

Table~\ref{tab:configuration} compares the Native/Fast and EB module
settings. EB evaluates the shared interface with a host-specific configuration.

\begin{table}[H]
\centering
\caption{Native/Fast and documented EB module configurations.
Exchange ratios bound attempts, not observed edit frequencies.}
\label{tab:configuration}
\small
\renewcommand{\arraystretch}{1.08}
\begin{tabularx}{\linewidth}{@{}lXX@{}}
\toprule
\rowcolor{TableHead}
Parameter & Native/Fast & EB\\
\midrule
\multicolumn{3}{@{}l}{\textit{Set selection}}\\
\texttt{reliability} & \texttt{margin} & \texttt{posterior}\\
\texttt{use\_dependency\_risk} & \texttt{true} & \texttt{true}\\
\texttt{lambda\_start/end} & \(0.03/0.01\) & \(0.03/0.01\)\\
\texttt{tau\_start/end} & \(16/4\) & \(16/4\)\\
\texttt{shortlist\_multiplier} & \(4.0\) & \(2.0\)\\
\texttt{min\_topk\_overlap} & \(0.75\) & \(0.90\)\\
\texttt{max\_penalty\_fraction} & \(0.10\) & \(0.10\)\\
\texttt{risk\_budget\_scale} & \texttt{null} & \texttt{null}\\
\texttt{structure\_protection} & \texttt{false} & \texttt{true}\\
\texttt{structure\_penalty} & \(0.02\) (inactive) & \(0.10\)\\
\addlinespace[3pt]
\multicolumn{3}{@{}l}{\textit{Remask exchange}}\\
\texttt{exchange} & \texttt{true} & \texttt{true}\\
\texttt{exchange\_ratio} & \(0.20\) & \(0.10\)\\
\texttt{min\_posterior\_gain} & \(0.002\) & \(0.005\)\\
\texttt{exchange\_window} & \([0.0,0.8]\) & \([0.10,0.60]\)\\
\texttt{exchange\_top\_k} & \(2\) & \(2\)\\
\bottomrule
\end{tabularx}
\end{table}

The public repository provides the Native/Fast and EB policy files,
LLaDA/Fast-dLLM generation entry points, and a release manifest. The
Native/Fast confidence floor is zero. Main comparisons and separately
executed sensitivity runs retain their own configuration IDs.

\FloatBarrier
\Needspace{8\baselineskip}
\section{Model and Evaluation Protocol}
\label{app:protocol}

The model roles and decoding routes are summarized below. Configuration
and execution details for the released LLaDA/Fast-dLLM runner are documented
in \texttt{REPRODUCE.md}; its checkpoint revision and release file hashes
are recorded in \texttt{RELEASE\_MANIFEST.json}.

\begin{table}[!htbp]
\centering
\caption{Model roles and evaluation routes. Transfer settings use
model-specific decoding; Fast adapter coverage is stated for the two
primary models. EB integration is documented separately.}
\label{tab:model_registry}
\small
\renewcommand{\arraystretch}{1.08}
\begin{tabularx}{\linewidth}{@{}llX@{}}
\toprule
\rowcolor{TableHead}
Model & Role & Decoding route\\
\midrule
LLaDA-8B & Primary & Uncached fixed; Fast fixed/adaptive; audited EB\\
Dream-7B-Instruct & Primary & Native entropy; Fast fixed/adaptive; separate EB runner\\
LLaDA2.1 & Transfer & Model-specific decoding\\
LLaDA2.2 & Transfer & Model-specific decoding\\
\bottomrule
\end{tabularx}
\end{table}

\subsection{Generation and Scoring}
\paragraph{Prompt and generation settings.}
Each model uses the template and generation rule appropriate to its
checkpoint and base/instruct variant.
The registry includes BOS handling, few-shot examples, generation length,
block size, temperature, token sampling, EOS, and task stop rules.
For LLaDA, the official block-32 native reference and uncached fixed-NFE
sweeps are distinct evaluation routes. Dream's native reference uses its
chat template, entropy sampler, temperature $0.1$, and top-$p=0.9$.
Its unmodified hook leaves native sampling intact, whereas the module
uses deterministic top-1 proposals from the same logits. Consequently,
these Dream native/Fast comparisons measure the complete decoding change,
including proposal generation, rather than a selector-only intervention.
EB follows its separate proposal and stopping protocol
(Appendix~\ref{app:eb}). Transfer models use their own routes without assuming
universal cache compatibility. Paired inputs and task scorers are shared.
Before scoring, outputs undergo token-level EOS truncation, task stop-string
truncation, and removal of special tokens, in that order.

\paragraph{Task scoring.}
GSM8K uses a versioned answer extractor; MATH-500 uses a versioned
mathematical equivalence procedure.
MBPP and HumanEval use the specified functional test suite and split.
Formatting validity is defined per task and reported separately.
Failed, truncated, and unparseable generations remain in the primary
denominator under the declared scoring policy.

\FloatBarrier
\subsection{Quality Reference and Operating-Point Selection}
\paragraph{Common native-128 quality reference.}
\label{app:qref}
Tables~\ref{tab:main}, \ref{tab:quality_margin}, and~\ref{tab:quality64}
use one shared native-128 observation per model--task pair. The reference
has 128 actual model evaluations, no cache, and \method{} disabled.
LLaDA uses the uncached fixed-NFE route; this is distinct from the
separate official block-32 reference route. Dream uses its unchanged
native entropy sampler, chat template, temperature 0.1, and top-$p=0.9$.
Task inputs, output processing, and scorers follow the protocol above.
We retain the reference realization used by Table~\ref{tab:quality64}
and do not mix in alternative-seed reference scores. This choice is
shared across all compared hosts and all quality margins.

Table~\ref{tab:native_reference} specifies each reference as native, module off, NFE 128,
on the named model and task. It lists the scored counts and both the
quality anchor and the main target to six decimals; operating-point
selection uses the full precision of $100c/n-3$, not the rounded display.
The reference registry and selected configuration identifiers accompany
the source, so the same anchors can be used to regenerate the tables.

\begin{table}[!htbp]
\centering
\caption{Shared native-128 quality references and three-point targets.
Selection uses unrounded task scores.}
\label{tab:native_reference}
\small
\begin{tabular*}{\linewidth}{@{\extracolsep{\fill}}llrrr@{}}
\toprule
Model & Task & Correct / total & $Q_{\mathrm{ref}}$ (\%) & $q_\star$ (\%)\\
\midrule
LLaDA-8B & GSM8K & 1003/1319 & 76.042456 & 73.042456\\
LLaDA-8B & MATH-500 & 177/500 & 35.400000 & 32.400000\\
LLaDA-8B & MBPP & 203/427 & 47.540984 & 44.540984\\
LLaDA-8B & HumanEval & 66/164 & 40.243902 & 37.243902\\
Dream-7B & GSM8K & 1059/1319 & 80.288097 & 77.288097\\
Dream-7B & MATH-500 & 218/500 & 43.600000 & 40.600000\\
Dream-7B & MBPP & 255/427 & 59.718970 & 56.718970\\
Dream-7B & HumanEval & 90/164 & 54.878049 & 51.878049\\
\bottomrule
\end{tabular*}
\end{table}

For example, the LLaDA/HumanEval target is 37.243902\%, whereas the
32-NFE candidate scores 37.195122\%. Both display as 37.2\% at one
decimal, but that candidate is ineligible. Native cost denominators in
Table~\ref{tab:main} are the fastest native points meeting the target;
they are distinct from the 128-NFE quality reference.

\paragraph{Quality--cost objective.}
For method $m$ and evaluated configurations $\mathcal H_m$, the
matched-quality cost is
\begin{equation}
 C_m(q_\star)=\min_{h\in\mathcal H_m:\,Q_m(h)\ge q_\star}L_m(h).
 \label{eq:quality_cost}
\end{equation}
The quality anchor and tolerance are fixed per model--task pair.
We report evaluation compression $N_B/N_{B+\method{}}$ separately from
the reduction fraction $1-N_{B+\method{}}/N_B$ and measured latency.
Requested steps, actual model evaluations, and partial recomputation
are recorded separately.

\paragraph{Configuration selection.}
Fixed-budget sweeps cover supported 16/32/64/128 settings and valid
intermediate points; adaptive hosts receive threshold sweeps.
Configurations are selected on development data and confirmed at the
common target on held-out data. The main quality margin is three
percentage points, with one- and two-point sensitivity results in
Appendix~\ref{app:analysis}.
The development registry records all examples and checkpoints used for
parameter selection or method design.
The 10\% development pilot is part of the selection data.
Reports on complete official test sets retain this selection provenance
and distinguish development examples from held-out confirmation.
Quality margins and candidate grids are fixed before final confirmation.

\subsection{Randomness and Uncertainty}
\paragraph{Sampling units.}
Stochastic decoding uses paired seeds and records the realized sampling
scheme. For deterministic argmax decoding, the sampling unit is the prompt;
seed labels do not create additional observations.
Bootstrap resampling clusters by prompt and retains joint seed and method
records. Main-text results report point estimates; Table 8 supplements the main table with 95\% confidence intervals for quality, NFE, and speedup. These intervals quantify
prompt-level and timing-repeat uncertainty, distinct from seed variation.

\FloatBarrier
\Needspace{8\baselineskip}
\section{Interventions, Ablations, and Runtime}
\label{app:analysis}

\subsection{Controlled Interventions and Ablations}
\paragraph{Source states and interventions.}
The reported same-state aggregate comparison covers LLaDA-8B/GSM8K
at 32 NFE and progress fractions \(0.125,0.25,0.5,0.75\).
Figure~\ref{fig:mechanism} subtracts the confidence-matched and
stratified-random continuation gains from the coordinated-selector
gain at each source progress. The underlying host-relative values
are preserved in Figure~\ref{fig:budget_sensitivity}.
Broader source-policy and budget coverage requires corresponding
state-level continuation records.

For a position \(i\) revealed using its current prediction, the conditional
response at \(j\) is
\[
 \Delta_{j\leftarrow i}
 =
 \mathrm{JS}\!\left(
 p_j(\cdot\mid x^{(t)}),
 p_j(\cdot\mid x^{(t)},i\leftarrow\hat x_i^{(t)})
 \right).
\]
This measures a model-distribution response; continuation experiments
separately measure final success.
It is a distribution-level extension of the absolute posterior-change
probe in Figure~\ref{fig:low_step} (right), with its metric reported separately.
Vocabulary truncation includes the residual probability mass.
Commitment-set branches share the source state, quota, current proposals,
continuation policy, and remaining budget.
Matching diagnostics summarize confidence, distance, and progress.
Offline probe evaluations are recorded separately from generation cost.

\paragraph{Ablation grid.}
The selection-by-remasking design is followed by spacing,
stratified-random, and shuffled-risk controls. Remasking withdraws
visible positions and supplies equally many additional fresh
commitments; it does not mean swapping two candidates within the
current masked set. Structural protection is disabled in the main
Native/Fast policy. An overlap value of one fixes the base position
set to the posterior baseline but does not disable remasking.
Figure~\ref{fig:mechanism} computes component contrasts from the same
32-NFE quality records as Table~\ref{tab:fixed_budget}, extended to
MATH-500 and MBPP. Its interaction term compares the remask increment
with versus without coordinated selection.

\paragraph{Frozen configurations and cost points.}
The matrices in Figure~\ref{fig:ablation} retain all 16 frozen/retuned
model--task pairs. Frozen gains are measured against their corresponding
hosts; retuning gains are direct quality differences within each pair.
The mean retuning increment averages all 16 primary and transfer
model--task cells. Development and held-out roles follow the selection
protocol in Appendix~\ref{app:protocol}.
The 16 deployment points use the full-policy and host-baseline
32-NFE ablation records for two primary models, four tasks, and
Native/Fast-fixed hosts. Their coordinates are
$100(L_{11}/L_{00}-1)$ and $Q_{11}-Q_{00}$; the latency includes signal
handling, selection, remask exchange, and applicable cache refresh.
The Dream proposal-generation qualification in
Appendix~\ref{app:protocol} applies to these pipeline comparisons.

\clearpage
\subsection{Timing and Resource Use}
\paragraph{Timing.}
Comparisons use warm-up, synchronized device timing, and interleaved repeats
on a dedicated L20 or fixed device group.
Every operating point retains generated outputs, actual evaluations,
effective length, latency, and peak memory.
A representative prompt set supports repeated timing, with complete-set
confirmation for key configurations.
Signal handling, selection, edits, and refresh have non-overlapping
instrumentation boundaries.
Partial-backbone recomputation is recorded as work and time.

\paragraph{Cross-server consistency.}
Both servers retain code, model, protocol, configuration, and sample IDs.
A complete paired comparison uses one matched environment.
The second server's EB scan supplies candidate configurations. Inclusion
in the common latency table requires paired reruns on the main timing
device. Repeated validation records retain their original sample
identities.

\paragraph{Additional experimental results.}
Figure~\ref{fig:systems} (top left) separates model time, module time,
and other overhead at the matched-quality operating points. Total latency
decreases from 7.49 to 5.92 s on the native host (21.0\%) and from 3.73
to 3.15 s on the cached host (15.6\%). Fewer model evaluations offset
the added coordination cost. The top-right bars are cumulative gains
over each host: selection yields 3.8/4.2 pp on full/cached forwards,
and selection plus exchange yields 5.0/5.8 pp. Thus exchange adds
1.2/1.6 pp beyond selection.

At fixed 32 NFE, increasing generation length from 256 to 2048 raises
relative overhead from 1.8\% to 2.2\% on full forwards and from 5.9\%
to 8.2\% under caching (bottom left). Caching reduces the host latency
denominator and adds refresh work, explaining the larger relative cost.
At length 1024, memory increases by 0.28/0.36 GiB for LLaDA native/cached
and 0.20/0.19 GiB for Dream native/cached (bottom right).

\begin{figure}[!htb]
\centering
\includegraphics[width=\linewidth]{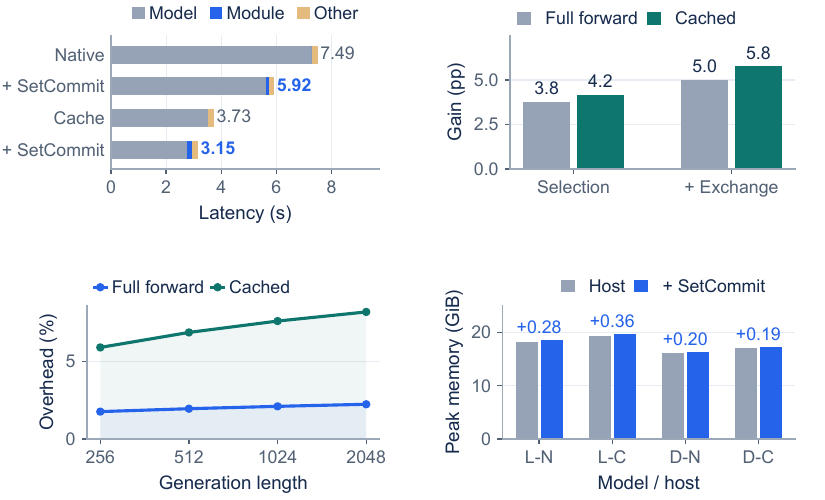}
\caption{ \textbf{Systems and robustness analyses.}
Top left: matched-quality LLaDA-8B/GSM8K runtime; module time includes signal
handling, selection, exchange, and cache refresh.
Top right: cumulative gains over each host for selection alone and
selection plus exchange, on LLaDA-8B/GSM8K at 32 NFE.
Bottom left: $100(L_{\mathrm{+SC}}/L_{\mathrm{host}}-1)$ across generation
lengths on the same model/task at 32 NFE.
Bottom right: peak memory at length 1024: L/D denote LLaDA-8B/Dream-7B and N/C
denote native/cached hosts. Labels above memory bars give added GiB.}
\label{fig:systems}
\end{figure}

\clearpage
\subsection{Quality Margins and Fixed-Budget Retention}
Table~\ref{tab:quality_margin} repeats operating-point selection at three
quality margins using the same sweeps. Table~\ref{tab:quality64} reports
fixed-64 quality retention relative to the shared native-128 reference.

\begin{table}[!htbp]
\centering
\caption{ \textbf{Quality-margin sensitivity.}
Geometric-mean incremental speedup across two primary models and four
tasks. Each target is reference quality minus the stated margin;
operating points are reselected from the same sweep.
Shading marks the main-table margin; entries are speedup factors.}
\label{tab:quality_margin}
\small
\begin{tabular*}{\linewidth}{@{\extracolsep{\fill}}lrr@{}}
\toprule
\rowcolor{TableHead}
Margin (pp) & Native + \method{} & Fast fixed + \method{}\\
\midrule
1 & \gaincell{1.68} & \gaincell{1.54}\\
2 & \gaincell{1.85} & \gaincell{1.71}\\
\rowcolor{SetCommitTint}3 & \gaincell{1.73} & \gaincell{1.48}\\
\bottomrule
\end{tabular*}
\end{table}

\begin{table}[!htbp]
\centering
\caption{ \textbf{Quality retention at exactly 64 evaluations.}
$Q$ is accuracy/pass@1 (\%). The reference is native sampling at 128 NFE
on the same model and task. Relative loss is
$100(Q_{128}-Q_{64})/Q_{128}$, distinct from the 3-point selection margin
in Table~\ref{tab:main}. Teal losses are within 5\%; bold $Q_{64}$ values
equal or exceed the reference. This same $Q_{128}$ defines
$Q_{\mathrm{ref}}$ in Table~\ref{tab:main}; losses use unrounded scores.}
\label{tab:quality64}
\small
\begin{tabular*}{\linewidth}{@{\extracolsep{\fill}}lrrrrr@{}}
\toprule
\rowcolor{TableHead}
Benchmark & Native & \multicolumn{2}{c}{Native + \method{}} &
\multicolumn{2}{c}{Fast fixed + \method{}}\\
\cmidrule(lr){3-4}\cmidrule(lr){5-6}
& $Q_{128}$ & $Q_{64}\uparrow$ & Loss (\%)$\downarrow$ & $Q_{64}\uparrow$ & Loss (\%)$\downarrow$\\
\midrule
\rowcolor{TableBand}
\multicolumn{6}{l}{\textbf{LLaDA-8B}}\\
GSM8K & 76.0 & \textbf{76.0} & \textcolor{SetCommitTeal}{0.00} & \textbf{76.0} & \textcolor{SetCommitTeal}{0.00}\\
MATH-500 & 35.4 & 34.6 & \textcolor{SetCommitTeal}{2.26} & 35.2 & \textcolor{SetCommitTeal}{0.56}\\
MBPP & 47.5 & \textbf{47.5} & \textcolor{SetCommitTeal}{0.00} & \textbf{47.5} & \textcolor{SetCommitTeal}{0.00}\\
HumanEval & 40.2 & 39.6 & \textcolor{SetCommitTeal}{1.52} & 39.6 & \textcolor{SetCommitTeal}{1.52}\\
\midrule
\rowcolor{TableBand}
\multicolumn{6}{l}{\textbf{Dream-7B}}\\
GSM8K & 80.3 & 80.1 & \textcolor{SetCommitTeal}{0.28} & 78.8 & \textcolor{SetCommitTeal}{1.89}\\
MATH-500 & 43.6 & 42.6 & \textcolor{SetCommitTeal}{2.29} & 43.2 & \textcolor{SetCommitTeal}{0.92}\\
MBPP & 59.7 & 58.8 & \textcolor{SetCommitTeal}{1.57} & 59.0 & \textcolor{SetCommitTeal}{1.18}\\
HumanEval & 54.9 & \textbf{54.9} & \textcolor{SetCommitTeal}{0.00} & 54.3 & \textcolor{SetCommitTeal}{1.11}\\
\bottomrule
\end{tabular*}
\end{table}

\clearpage
\subsection{Budget and Parameter Sensitivity}
Figure~\ref{fig:budget_sensitivity} (top left) shows GSM8K gains over
native decoding at matched NFE. Gains are largest at 16 evaluations
(10.6 pp for LLaDA and 13.0 pp for Dream), fall to 5.0/5.8 pp at 32,
and approach zero at 128. This concentrates the observed benefit in the
few-step regime.

The top-right panel reports host-relative continuation gains:
\method{} gives 5.99/4.12/1.63/1.84 pp at 12.5/25/50/75\% progress.
Subtracting the matched control leaves 1.86/0.88/0.07/0.64 pp
(Figure~\ref{fig:mechanism}), isolating the gain over that control.

The risk-weight scan uses 0/0.5/1/1.5/2 times the nominal risk strength;
the remask-ratio scan uses 0/0.05/0.10/0.20/0.30.
These sensitivity runs are executed separately from the main-policy
evaluation. The ratio bounds
attempted withdrawals relative to the net quota; eligibility, progress,
posterior-gain, and set-risk checks determine realized remasking.
The frozen Native/Fast ratio is 0.20
(Table~\ref{tab:configuration}).
Risk weights 0/0.5/1/1.5/2 yield 67.78/71.19/71.42/70.96/70.36\%
quality, showing a nonmonotonic tradeoff. Remask caps 0/0.05/0.10/0.20/0.30 yield
70.20/71.11/71.42/71.11/70.96\%; the 0.10 peak exceeds 0.20 by only
0.30 pp before rounding. These independent results are point estimates.
The cap bounds attempted withdrawals; the acceptance checks determine
the realized remasking rate.

\begin{figure}[!htb]
\centering
\includegraphics[width=\linewidth]{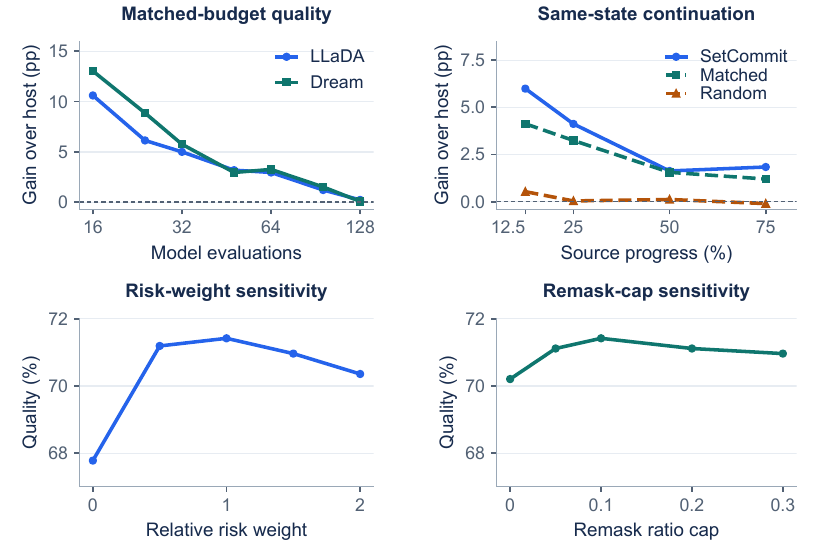}
\caption{\textbf{Budget, continuation, and parameter analyses.}
Top left: GSM8K quality gains over native hosts at matched NFE.
Top right: LLaDA-8B/GSM8K same-state continuation gains relative to
the host at 32 NFE, before forming the control differences in
Figure~\ref{fig:mechanism}. Bottom: independent relative risk-weight
and remask-ratio-cap sweeps on LLaDA-8B/GSM8K at 32 NFE; all evaluated
settings are shown with a common quality scale. These curves do not
mark a selected default.}
\label{fig:budget_sensitivity}
\end{figure}

\clearpage
\section{Main-Table Uncertainty Supplement}
\label{app:main_uncertainty}
CIs use 10{,}000 prompt-cluster bootstrap replicates for quality and mean NFE, retaining paired method/seed records. Seed variation requires separate reporting; deterministic reruns add no quality samples. Timing CIs resample matched rounds on a fixed prompt set, recomputing ratios and four-task geometric means. Intervals condition on the selected settings and reference, without configuration reselection or a quality-equivalence claim.
\begingroup
\fontsize{8.5}{10}\selectfont
\setlength{\tabcolsep}{5pt}
\renewcommand{\arraystretch}{1.04}
\setlength{\LTcapwidth}{\textwidth}
\begin{longtable}{@{}lccc@{}}
\caption{\textbf{Uncertainty supplement to Table~\ref{tab:main}.}
Entries: estimate [95\% CI]; points are copied or derived from Table 1. Timing design: 3 warm-up and 10 paired measured rounds per setting. Timing CIs cover repeat variation only. Fixed-budget NFE and native self-ratios have degenerate intervals.}
\label{tab:main_uncertainty}\\
\endfirsthead
\caption[]{\textbf{Uncertainty supplement to Table~\ref{tab:main} (continued).}
Entries: estimate [95\% CI].}\\
\endhead
\multicolumn{4}{r}{\scriptsize Continued on next page.}\\
\endfoot
\endlastfoot
\toprule
\rowcolor{TableBand}
\multicolumn{4}{l}{\textbf{A. LLaDA-8B: all main-table configurations}}\\
\rowcolor{TableHead}
Decoder & $Q$ (\%) & Mean NFE & Total speedup ($\times$)\\
\midrule
\rowcolor{TableHead}
\multicolumn{4}{l}{\textit{GSM8K} ($n=1319$)}\\
Native & $73.1\;[70.7,75.4]$ & $64.0\;[64.0,64.0]$ & $1.00\;[1.00,1.00]$\\
\rowcolor{SetCommitTint}\quad\textcolor{SetCommitBlue}{+ \method{}} & $74.2\;[71.8,76.6]$ & $48.0\;[48.0,48.0]$ & $1.27\;[1.25,1.29]$\\
Fast fixed & $73.1\;[70.7,75.4]$ & $64.0\;[64.0,64.0]$ & $2.01\;[1.96,2.06]$\\
\rowcolor{SetCommitTint}\quad\textcolor{SetCommitBlue}{+ \method{}} & $74.6\;[72.3,77.0]$ & $48.0\;[48.0,48.0]$ & $2.38\;[2.31,2.45]$\\
Fast adaptive & $73.6\;[71.2,75.9]$ & $62.2\;[61.7,62.7]$ & $2.05\;[1.98,2.12]$\\
\rowcolor{SetCommitTint}\quad\textcolor{SetCommitBlue}{+ \method{}} & $73.1\;[70.7,75.4]$ & $32.3\;[32.0,32.6]$ & $3.49\;[3.34,3.65]$\\
EB-Sampler & $73.6\;[71.2,76.0]$ & $58.5\;[57.9,59.1]$ & $1.09\;[1.05,1.13]$\\
\rowcolor{SetCommitTint}\quad\textcolor{SetCommitBlue}{+ \method{}} & $75.7\;[73.3,77.9]$ & $47.7\;[47.2,48.2]$ & $1.33\;[1.29,1.37]$\\
DAWN & $74.0\;[71.6,76.3]$ & $60.1\;[59.4,60.8]$ & $1.04\;[1.01,1.07]$\\
DOS & $73.5\;[71.0,75.9]$ & $59.9\;[59.2,60.6]$ & $1.05\;[1.03,1.07]$\\
\addlinespace[2pt]
\rowcolor{TableHead}
\multicolumn{4}{l}{\textit{MATH-500} ($n=500$)}\\
Native & $34.2\;[30.0,38.4]$ & $96.0\;[96.0,96.0]$ & $1.00\;[1.00,1.00]$\\
\rowcolor{SetCommitTint}\quad\textcolor{SetCommitBlue}{+ \method{}} & $33.8\;[29.6,38.0]$ & $48.0\;[48.0,48.0]$ & $1.90\;[1.86,1.93]$\\
Fast fixed & $34.0\;[29.8,38.2]$ & $96.0\;[96.0,96.0]$ & $1.96\;[1.92,2.01]$\\
\rowcolor{SetCommitTint}\quad\textcolor{SetCommitBlue}{+ \method{}} & $34.6\;[30.4,38.8]$ & $48.0\;[48.0,48.0]$ & $3.72\;[3.63,3.81]$\\
Fast adaptive & $32.4\;[28.2,36.6]$ & $64.8\;[63.8,65.8]$ & $2.99\;[2.92,3.06]$\\
\rowcolor{SetCommitTint}\quad\textcolor{SetCommitBlue}{+ \method{}} & $33.4\;[29.4,37.6]$ & $48.0\;[47.2,48.8]$ & $3.87\;[3.78,3.97]$\\
EB-Sampler & $32.4\;[28.4,36.6]$ & $62.9\;[61.8,64.1]$ & $1.53\;[1.50,1.56]$\\
\rowcolor{SetCommitTint}\quad\textcolor{SetCommitBlue}{+ \method{}} & $33.6\;[29.4,37.6]$ & $44.6\;[43.8,45.5]$ & $2.12\;[2.06,2.19]$\\
DAWN & $33.2\;[29.0,37.4]$ & $64.7\;[63.4,66.0]$ & $1.39\;[1.35,1.42]$\\
DOS & $32.4\;[28.2,36.4]$ & $63.0\;[61.7,64.3]$ & $1.45\;[1.42,1.49]$\\
\addlinespace[2pt]
\rowcolor{TableHead}
\multicolumn{4}{l}{\textit{MBPP} ($n=427$)}\\
Native & $46.6\;[41.9,51.3]$ & $96.0\;[96.0,96.0]$ & $1.00\;[1.00,1.00]$\\
\rowcolor{SetCommitTint}\quad\textcolor{SetCommitBlue}{+ \method{}} & $46.1\;[41.5,50.8]$ & $48.0\;[48.0,48.0]$ & $1.94\;[1.89,1.99]$\\
Fast fixed & $44.7\;[40.0,49.4]$ & $64.0\;[64.0,64.0]$ & $2.93\;[2.80,3.06]$\\
\rowcolor{SetCommitTint}\quad\textcolor{SetCommitBlue}{+ \method{}} & $46.8\;[42.2,51.5]$ & $48.0\;[48.0,48.0]$ & $3.47\;[3.35,3.59]$\\
Fast adaptive & $44.7\;[40.0,49.4]$ & $58.3\;[57.4,59.2]$ & $3.22\;[3.12,3.34]$\\
\rowcolor{SetCommitTint}\quad\textcolor{SetCommitBlue}{+ \method{}} & $45.4\;[40.7,50.1]$ & $44.9\;[44.0,45.8]$ & $3.76\;[3.66,3.86]$\\
EB-Sampler & $45.2\;[40.5,50.1]$ & $63.5\;[62.2,64.8]$ & $1.50\;[1.44,1.57]$\\
\rowcolor{SetCommitTint}\quad\textcolor{SetCommitBlue}{+ \method{}} & $45.9\;[41.5,50.8]$ & $45.8\;[44.9,46.7]$ & $1.97\;[1.91,2.03]$\\
DAWN & $45.2\;[40.5,49.9]$ & $60.5\;[59.3,61.7]$ & $1.53\;[1.48,1.59]$\\
DOS & $45.4\;[40.7,50.1]$ & $62.7\;[61.2,64.2]$ & $1.50\;[1.45,1.56]$\\
\addlinespace[2pt]
\rowcolor{TableHead}
\multicolumn{4}{l}{\textit{HumanEval} ($n=164$)}\\
Native & $39.0\;[31.7,46.3]$ & $96.0\;[96.0,96.0]$ & $1.00\;[1.00,1.00]$\\
\rowcolor{SetCommitTint}\quad\textcolor{SetCommitBlue}{+ \method{}} & $39.0\;[31.7,46.3]$ & $48.0\;[48.0,48.0]$ & $1.89\;[1.87,1.93]$\\
Fast fixed & $39.6\;[32.3,47.6]$ & $96.0\;[96.0,96.0]$ & $1.91\;[1.86,1.97]$\\
\rowcolor{SetCommitTint}\quad\textcolor{SetCommitBlue}{+ \method{}} & $38.4\;[31.1,45.7]$ & $48.0\;[48.0,48.0]$ & $3.36\;[3.26,3.46]$\\
Fast adaptive & $39.0\;[31.7,46.3]$ & $93.4\;[90.9,96.0]$ & $1.99\;[1.95,2.02]$\\
\rowcolor{SetCommitTint}\quad\textcolor{SetCommitBlue}{+ \method{}} & $38.4\;[31.1,45.7]$ & $44.8\;[43.5,46.1]$ & $3.64\;[3.42,3.87]$\\
EB-Sampler & $39.6\;[32.3,47.0]$ & $92.6\;[89.6,95.5]$ & $1.01\;[0.99,1.03]$\\
\rowcolor{SetCommitTint}\quad\textcolor{SetCommitBlue}{+ \method{}} & $39.0\;[31.7,46.4]$ & $46.4\;[44.8,48.0]$ & $1.95\;[1.88,2.03]$\\
DAWN & $38.4\;[31.1,45.7]$ & $60.9\;[58.7,63.1]$ & $1.44\;[1.39,1.49]$\\
DOS & $39.6\;[32.3,47.0]$ & $97.4\;[94.0,100.9]$ & $0.96\;[0.94,0.99]$\\
\bottomrule
\newpage
\toprule
\rowcolor{TableBand}
\multicolumn{4}{l}{\textbf{B. Dream-7B: all main-table configurations}}\\
\rowcolor{TableHead}
Decoder & $Q$ (\%) & Mean NFE & Total speedup ($\times$)\\
\midrule
\rowcolor{TableHead}
\multicolumn{4}{l}{\textit{GSM8K} ($n=1319$)}\\
Native & $78.9\;[76.6,81.0]$ & $96.0\;[96.0,96.0]$ & $1.00\;[1.00,1.00]$\\
\rowcolor{SetCommitTint}\quad\textcolor{SetCommitBlue}{+ \method{}} & $78.6\;[76.4,80.8]$ & $48.0\;[48.0,48.0]$ & $1.96\;[1.93,1.99]$\\
Fast fixed & $77.4\;[75.1,79.7]$ & $64.0\;[64.0,64.0]$ & $2.89\;[2.83,2.94]$\\
\rowcolor{SetCommitTint}\quad\textcolor{SetCommitBlue}{+ \method{}} & $78.0\;[75.7,80.2]$ & $48.0\;[48.0,48.0]$ & $3.56\;[3.48,3.64]$\\
Fast adaptive & $77.9\;[75.7,80.1]$ & $65.3\;[64.8,65.8]$ & $2.97\;[2.90,3.04]$\\
\rowcolor{SetCommitTint}\quad\textcolor{SetCommitBlue}{+ \method{}} & $77.6\;[75.4,79.9]$ & $31.2\;[30.9,31.5]$ & $5.47\;[5.32,5.63]$\\
EB-Sampler & $78.2\;[75.9,80.4]$ & $61.9\;[61.3,62.5]$ & $1.52\;[1.47,1.57]$\\
\rowcolor{SetCommitTint}\quad\textcolor{SetCommitBlue}{+ \method{}} & $77.7\;[75.4,79.9]$ & $31.6\;[31.3,31.9]$ & $2.89\;[2.77,3.01]$\\
DAWN & $78.2\;[75.9,80.4]$ & $60.7\;[60.0,61.3]$ & $1.52\;[1.49,1.55]$\\
DOS & $78.2\;[75.9,80.4]$ & $61.8\;[61.0,62.6]$ & $1.48\;[1.45,1.52]$\\
\addlinespace[2pt]
\rowcolor{TableHead}
\multicolumn{4}{l}{\textit{MATH-500} ($n=500$)}\\
Native & $42.6\;[38.4,47.0]$ & $96.0\;[96.0,96.0]$ & $1.00\;[1.00,1.00]$\\
\rowcolor{SetCommitTint}\quad\textcolor{SetCommitBlue}{+ \method{}} & $41.2\;[37.0,45.6]$ & $48.0\;[48.0,48.0]$ & $1.92\;[1.89,1.96]$\\
Fast fixed & $42.8\;[38.4,47.2]$ & $96.0\;[96.0,96.0]$ & $2.01\;[1.97,2.05]$\\
\rowcolor{SetCommitTint}\quad\textcolor{SetCommitBlue}{+ \method{}} & $42.8\;[38.4,47.0]$ & $48.0\;[48.0,48.0]$ & $3.77\;[3.66,3.87]$\\
Fast adaptive & $42.4\;[38.2,46.6]$ & $92.2\;[91.0,93.4]$ & $2.09\;[2.03,2.15]$\\
\rowcolor{SetCommitTint}\quad\textcolor{SetCommitBlue}{+ \method{}} & $41.8\;[37.6,46.0]$ & $46.7\;[45.9,47.5]$ & $3.96\;[3.83,4.09]$\\
EB-Sampler & $41.0\;[36.8,45.2]$ & $64.1\;[63.0,65.3]$ & $1.52\;[1.49,1.54]$\\
\rowcolor{SetCommitTint}\quad\textcolor{SetCommitBlue}{+ \method{}} & $42.0\;[37.6,46.2]$ & $44.8\;[43.9,45.7]$ & $2.10\;[2.04,2.16]$\\
DAWN & $41.0\;[36.8,45.4]$ & $64.8\;[63.7,65.9]$ & $1.44\;[1.41,1.47]$\\
DOS & $41.0\;[36.8,45.4]$ & $62.7\;[61.3,64.1]$ & $1.50\;[1.46,1.55]$\\
\addlinespace[2pt]
\rowcolor{TableHead}
\multicolumn{4}{l}{\textit{MBPP} ($n=427$)}\\
Native & $57.4\;[52.7,62.1]$ & $64.0\;[64.0,64.0]$ & $1.00\;[1.00,1.00]$\\
\rowcolor{SetCommitTint}\quad\textcolor{SetCommitBlue}{+ \method{}} & $58.3\;[53.6,63.0]$ & $48.0\;[48.0,48.0]$ & $1.25\;[1.22,1.28]$\\
Fast fixed & $57.6\;[52.9,62.3]$ & $64.0\;[64.0,64.0]$ & $1.92\;[1.86,1.98]$\\
\rowcolor{SetCommitTint}\quad\textcolor{SetCommitBlue}{+ \method{}} & $57.6\;[52.9,62.3]$ & $32.0\;[32.0,32.0]$ & $3.39\;[3.28,3.51]$\\
Fast adaptive & $57.6\;[52.9,62.3]$ & $59.3\;[58.4,60.2]$ & $2.15\;[2.09,2.22]$\\
\rowcolor{SetCommitTint}\quad\textcolor{SetCommitBlue}{+ \method{}} & $59.0\;[54.3,63.7]$ & $43.9\;[43.0,44.8]$ & $2.59\;[2.46,2.72]$\\
EB-Sampler & $58.1\;[53.4,62.8]$ & $60.5\;[59.2,61.8]$ & $1.07\;[1.04,1.10]$\\
\rowcolor{SetCommitTint}\quad\textcolor{SetCommitBlue}{+ \method{}} & $57.4\;[52.7,62.1]$ & $32.4\;[31.7,33.1]$ & $1.83\;[1.76,1.90]$\\
DAWN & $58.3\;[53.6,63.0]$ & $62.7\;[61.5,63.9]$ & $0.99\;[0.98,1.01]$\\
DOS & $57.8\;[53.2,62.5]$ & $60.7\;[59.1,62.3]$ & $1.03\;[0.98,1.07]$\\
\addlinespace[2pt]
\rowcolor{TableHead}
\multicolumn{4}{l}{\textit{HumanEval} ($n=164$)}\\
Native & $54.3\;[47.0,62.2]$ & $96.0\;[96.0,96.0]$ & $1.00\;[1.00,1.00]$\\
\rowcolor{SetCommitTint}\quad\textcolor{SetCommitBlue}{+ \method{}} & $53.7\;[46.3,61.0]$ & $48.0\;[48.0,48.0]$ & $1.93\;[1.89,1.96]$\\
Fast fixed & $52.4\;[45.1,59.8]$ & $64.0\;[64.0,64.0]$ & $2.83\;[2.76,2.89]$\\
\rowcolor{SetCommitTint}\quad\textcolor{SetCommitBlue}{+ \method{}} & $53.7\;[46.3,61.0]$ & $48.0\;[48.0,48.0]$ & $3.35\;[3.27,3.43]$\\
Fast adaptive & $54.3\;[47.0,61.6]$ & $93.5\;[90.4,96.7]$ & $2.11\;[2.06,2.17]$\\
\rowcolor{SetCommitTint}\quad\textcolor{SetCommitBlue}{+ \method{}} & $53.7\;[45.7,61.0]$ & $45.6\;[44.1,47.1]$ & $3.58\;[3.44,3.71]$\\
EB-Sampler & $53.0\;[45.7,61.0]$ & $63.2\;[61.1,65.4]$ & $1.49\;[1.44,1.54]$\\
\rowcolor{SetCommitTint}\quad\textcolor{SetCommitBlue}{+ \method{}} & $52.4\;[44.5,59.8]$ & $58.6\;[56.6,60.6]$ & $1.56\;[1.48,1.63]$\\
DAWN & $52.4\;[44.5,59.8]$ & $61.1\;[58.7,63.5]$ & $1.48\;[1.43,1.53]$\\
DOS & $52.4\;[45.1,60.4]$ & $63.2\;[60.6,65.9]$ & $1.44\;[1.40,1.50]$\\
\bottomrule
\newpage
\toprule
\rowcolor{TableBand}
\multicolumn{4}{l}{\textbf{C. Paired contrasts at the selected operating points}}\\
\rowcolor{TableHead}
Host $B$ & $\Delta Q$ (pp) & $S_{\mathrm{inc}}$ ($\times$) & Timing rounds\\
\midrule
\rowcolor{TableHead}
\multicolumn{4}{l}{\textit{LLaDA-8B / GSM8K}}\\
Native & $+1.1\;[-0.5,+2.8]$ & $1.27\;[1.25,1.29]$ & 10\\
Fast fixed & $+1.5\;[-0.2,+3.2]$ & $1.18\;[1.15,1.22]$ & 10\\
Fast adaptive & $-0.5\;[-2.4,+1.4]$ & $1.70\;[1.64,1.76]$ & 10\\
EB-Sampler & $+2.0\;[+0.1,+3.9]$ & $1.22\;[1.19,1.25]$ & 10\\
\addlinespace[2pt]
\rowcolor{TableHead}
\multicolumn{4}{l}{\textit{LLaDA-8B / MATH-500}}\\
Native & $-0.4\;[-3.2,+2.4]$ & $1.90\;[1.86,1.93]$ & 10\\
Fast fixed & $+0.6\;[-2.4,+3.6]$ & $1.90\;[1.84,1.96]$ & 10\\
Fast adaptive & $+1.0\;[-2.2,+4.2]$ & $1.29\;[1.25,1.34]$ & 10\\
EB-Sampler & $+1.2\;[-2.2,+4.4]$ & $1.39\;[1.34,1.43]$ & 10\\
\addlinespace[2pt]
\rowcolor{TableHead}
\multicolumn{4}{l}{\textit{LLaDA-8B / MBPP}}\\
Native & $-0.5\;[-3.5,+2.6]$ & $1.94\;[1.89,1.99]$ & 10\\
Fast fixed & $+2.1\;[-1.2,+5.6]$ & $1.18\;[1.15,1.22]$ & 10\\
Fast adaptive & $+0.7\;[-3.0,+4.2]$ & $1.17\;[1.13,1.21]$ & 10\\
EB-Sampler & $+0.7\;[-2.8,+4.0]$ & $1.31\;[1.27,1.36]$ & 10\\
\addlinespace[2pt]
\rowcolor{TableHead}
\multicolumn{4}{l}{\textit{LLaDA-8B / HumanEval}}\\
Native & $+0.0\;[-5.5,+5.5]$ & $1.89\;[1.87,1.93]$ & 10\\
Fast fixed & $-1.2\;[-6.7,+4.3]$ & $1.76\;[1.69,1.82]$ & 10\\
Fast adaptive & $-0.6\;[-6.7,+5.5]$ & $1.83\;[1.73,1.94]$ & 10\\
EB-Sampler & $-0.6\;[-6.7,+5.5]$ & $1.93\;[1.85,2.04]$ & 10\\
\addlinespace[2pt]
\rowcolor{TableHead}
\multicolumn{4}{l}{\textit{Dream-7B / GSM8K}}\\
Native & $-0.3\;[-1.9,+1.3]$ & $1.96\;[1.93,1.99]$ & 10\\
Fast fixed & $+0.6\;[-1.2,+2.4]$ & $1.23\;[1.20,1.27]$ & 10\\
Fast adaptive & $-0.3\;[-2.3,+1.7]$ & $1.84\;[1.77,1.91]$ & 10\\
EB-Sampler & $-0.5\;[-2.4,+1.4]$ & $1.90\;[1.82,1.97]$ & 10\\
\addlinespace[2pt]
\rowcolor{TableHead}
\multicolumn{4}{l}{\textit{Dream-7B / MATH-500}}\\
Native & $-1.4\;[-4.0,+1.2]$ & $1.92\;[1.89,1.96]$ & 10\\
Fast fixed & $+0.0\;[-3.2,+3.2]$ & $1.88\;[1.83,1.92]$ & 10\\
Fast adaptive & $-0.6\;[-4.0,+2.6]$ & $1.89\;[1.81,1.97]$ & 10\\
EB-Sampler & $+1.0\;[-2.2,+4.2]$ & $1.38\;[1.35,1.42]$ & 10\\
\addlinespace[2pt]
\rowcolor{TableHead}
\multicolumn{4}{l}{\textit{Dream-7B / MBPP}}\\
Native & $+0.9\;[-2.3,+4.4]$ & $1.25\;[1.22,1.28]$ & 10\\
Fast fixed & $+0.0\;[-3.5,+3.3]$ & $1.77\;[1.71,1.84]$ & 10\\
Fast adaptive & $+1.4\;[-2.1,+4.9]$ & $1.20\;[1.14,1.28]$ & 10\\
EB-Sampler & $-0.7\;[-4.5,+3.3]$ & $1.71\;[1.64,1.78]$ & 10\\
\addlinespace[2pt]
\rowcolor{TableHead}
\multicolumn{4}{l}{\textit{Dream-7B / HumanEval}}\\
Native & $-0.6\;[-6.1,+4.9]$ & $1.93\;[1.89,1.96]$ & 10\\
Fast fixed & $+1.2\;[-4.3,+6.7]$ & $1.18\;[1.15,1.22]$ & 10\\
Fast adaptive & $-0.6\;[-6.7,+4.9]$ & $1.70\;[1.62,1.78]$ & 10\\
EB-Sampler & $-0.6\;[-6.7,+5.5]$ & $1.05\;[1.00,1.10]$ & 10\\
\addlinespace[2pt]
\midrule
\rowcolor{TableBand}
\multicolumn{4}{l}{\textbf{D. Four-task geometric-mean incremental speedup (Mean extra)}}\\
\rowcolor{TableHead}
Host $B$ & LLaDA-8B ($\times$) & Dream-7B ($\times$) & Rounds/task\\
\midrule
Native & $1.72\;[1.71,1.74]$ & $1.74\;[1.72,1.75]$ & 10\\
Fast fixed & $1.47\;[1.45,1.49]$ & $1.48\;[1.46,1.50]$ & 10\\
Fast adaptive & $1.47\;[1.45,1.50]$ & $1.63\;[1.59,1.68]$ & 10\\
EB-Sampler & $1.44\;[1.41,1.47]$ & $1.47\;[1.45,1.50]$ & 10\\
\bottomrule
\end{longtable}
\endgroup
\noindent\footnotesize In panel C, $\Delta Q=Q_{B+\method{}}-Q_B$ and
$S_{\mathrm{inc}}=L_B/L_{B+\method{}}$. Quality differences use task-score counts
consistent with the displayed percentages; differences of rounded cells may vary
by 0.1 pp. These are matched-target comparisons, not fixed-NFE ablations.
\normalsize

\end{document}